\documentclass[a4paper, 11pt]{article}
\usepackage[english]{babel}
\usepackage[T1]{fontenc}

\usepackage{a4wide}
\usepackage{amsmath,amssymb,amsthm}
\usepackage{graphicx,latexsym,color,hyperref}
\usepackage{newpxtext, newpxmath}
\usepackage[title]{appendix}
\usepackage{rotating, comment, bm}
\usepackage[ruled,vlined,linesnumbered]{algorithm2e}
\usepackage{enumitem}
\usepackage{cleveref}
\usepackage{caption}
\usepackage{svg}
\usepackage{subcaption}
\usepackage[numbers]{natbib}
\usepackage[normalem]{ulem}

\newif\ifdraft

\drafttrue

\definecolor{ForestGreen}{RGB}{34,139,34}

\usepackage{xspace}
\newcommand{\epoch}{epoch\xspace}
\newcommand{\runningtime}{cost\xspace}
\newcommand{\cost}{cost\xspace}
\newcommand*{\score}{score\xspace}
\newcommand*{\cloud}{\mathcal{C}}
\newcommand*{\grid}{\mathcal{G}}
\SetKwRepeat{KwDo}{do}{while}
\SetKwInOut{KwReturn}{Return}
\SetKwInOut{KwInit}{Initialization}
\newcommand{\Kalman}{\texttt{Kalman}}
\newcommand{\regress}{\texttt{Lambda}}
\newcommand{\OTF}{\texttt{OTF}}

\newcommand{\algName}{AutoBCT\xspace}
\newcommand{\cM}{\mathcal{M}}

\theoremstyle{plain}
\newtheorem{theorem}{Theorem}[section]

\newtheorem{lemma}[theorem]{Lemma}

\theoremstyle{remark}
\newtheorem{remark}[theorem]{Remark}

\theoremstyle{definition}
\newtheorem{definition}[theorem]{Definition}

\theoremstyle{remark}
\newtheorem{XxmpX}{Example} 
  {%
   \pushQED{\qed}\begin{XxmpX}}
  {\popQED\end{XxmpX}}

  \makeatletter
\newcommand\assumptionlabel[1]{\hspace\labelsep
                               \normalfont\bfseries #1\ \ \gdef\@currentlabel{#1}}
\newenvironment{assumption}
               {\smallskip\list{}{\labelwidth\z@ \itemindent-\leftmargin
                        }}
               {\endlist}
\makeatother

\newcounter{AlgorithmJP}[section]

\renewcommand{\theAlgorithmJP}{\thesection.\arabic{AlgorithmJP}}

\def\cA{\mathcal{A}}
\def\cB{\mathcal{B}}

\def\cT{\mathcal{T}}
\def\Tcost{\Upsilon}

\def\0{\mathbf{0}}
\def\er{\mathbb{R}}
\def\prob{\mathbb{P}}
\def\ee{E}
\def\X{\mathcal{X}}

\def\argmin{\mathop{\rm argmin}}
\def\argmax{\mathop{\rm arg\, max}}

\def\ef{\mathcal{F}}

\def\wtl{\widetilde}

\def\eop{\hfill{$\Box$}\medskip}

\newcommand\ind[1]{{1\kern-0.4em 1}_{\{#1\}}}
\newcommand\U{\mathcal{U}}
\newcommand\balpha{\bm{\alpha}}
\newcommand\bbeta{\bm{\beta}}
\newcommand\bphi{\bm{\phi}}
\newcommand\bpsi{\bm{\psi}}

\title{Automatic model training under restrictive time constraints\footnote{Support of EPSRC grants EP/N013980/1 and EP/N510129/1 is gratefully acknowledged.}}
\author{
Lukas Cironis\footnote{School of Mathematics, University of Leeds, Leeds LS2 9JT, UK (e-mail: L.Cironis@leeds.ac.uk)}
\and
Jan Palczewski\footnote{School of Mathematics, University of Leeds, Leeds LS2 9JT, UK (e-mail: J.Palczewski@leeds.ac.uk)}
\and
Georgios Aivaliotis\footnote{School of Mathematics, University of Leeds, Leeds LS2 9JT, UK and The Alan Turing Institute, The British Library, 2QR, 96 Euston Rd, London NW1 2DB(e-mail: G.Aivaliotis@leeds.ac.uk)}
}

\date{\today}

\begin{document}

\maketitle

\begin{abstract}
We develop a hyperparameter optimisation algorithm, Automated Budget Constrained Training (\algName), which balances the quality of a model with the computational cost required to tune it. The relationship between hyperparameters, model quality and computational cost must be learnt and this learning is incorporated directly into the optimisation problem. At each training epoch, the algorithm decides whether to terminate or continue training, and, in the latter case, what values of hyperparameters to use. This decision weighs \emph{optimally} potential improvements in the quality with the additional training time and the uncertainty about the learnt quantities. The performance of our algorithm is verified on a number of machine learning problems encompassing random forests and neural networks.
Our approach is rooted in the theory of Markov decision processes with partial information and we develop a numerical method to compute the value function and an optimal strategy.
\end{abstract}

\section{Introduction}

The emergence of Machine Learning (ML) methods as an effective and easy to use tool for modeling and prediction has opened up new horizons for users in all aspects of business, finance, health, and research \citep{ghoddusi2019machine, briganti2020artificial}. In stark contrast to more traditional statistical methods that require relevant expertise, ML methods constitute largely black-boxes and are based on minimal assumptions. This resulted in adoption of ML from an audience of users with possible expertise in the domain of application but little or no expertise in statistics or computer science. Despite these features, effective use of ML does require good knowledge of the  method used. The choice of the method (algorithm), the tuning of its hyperparameters and the architecture requires experience and good understanding of those methods and the data. Naturally, trial and error might provide possible solutions; however it can also result in sub-optimal use of ML and wasted computational resources.

The goal of this research is to find optimal values of hyperparameters for a given dataset and a modelling objective given that the relationship between the hyperparameters, model quality (\emph{model score}) and training cost is not known in advance. This problem is akin to the classical AutoML setup \citep{autoMLbook2019} with one crucial difference: if desired, the system has to produce results in radically shorter time than classical AutoML solutions. This comes at the expense of accuracy, so taking the trade-off between the model quality and the running time explicitly into account lies at the heart of this paper. The relationship between hyperparameters, model quality and computational cost for a particular modelling problem and a particular dataset must be learnt and this learning is incorporated directly into the optimisation problem.

\subsection{Main Contribution}

We propose Automated Budget Constrained Training (\algName), an algorithm based on theory of Markov Decision Processes with partially observable dynamics, combining stochastic control, optimal stopping and filtering. In \algName we treat hyperparameters as a control variable and introduce an objective function that balances model \score with the \cost, where both are often functions of the accuracy and the running time, respectively. The optimisation problem takes the form
\begin{equation}\label{eqn:2}
\sup_{(u_n), \tau} \ee \Big[h_\tau - \gamma \sum_{n=1}^{\tau} t_n\Big],
\end{equation}
where the optimisation is over $\tau \ge 1$ and $(u_n)_{n=0}^\infty$. Here, $\tau$ is the number of training epochs (a quantity dependent on the course of learning/training), $u_{n} \in \U$ denotes the choice of hyperparameters at epoch $n$ (also dependent on the history), and $\U$ is the set of admissible hyperparameters. The quantity $h_n$ is the model \score corresponding to the choice of hyperparameters $u_{n-1}$. The score $h_n$ may be random as it depends on the training process of the model (which can itself include randomness as in the construction of random forests or in the stochastic gradient method for neural networks) and on the choice of training and validation datasets (including K-fold cross-validation). The computational cost $t_n$ also depends on the choice of hyperparameters $u_{n-1}$ and can be random due to random variations in running time commonly observed in practical applications, arising both from a hardware and an operating system's side as well as from random choices in the training process.

In the objective function we consider the expectation of the score at epoch $\tau$ and the cumulative computational cost up to epoch $\tau$ since the randomness in observation of those quantities can be viewed as external to the problem and not under our control. The constant $\gamma>0$ models our aversion to computational cost. The higher the number the more sensitive we are to increase in cost and willing to accept models with lower scores.

The choice of hyperparameters $u_{n-1}$ for the epoch $n$ is based on the observation of model scores $h_1, \ldots, h_{n-1}$ and costs $t_1, \ldots, t_{n-1}$ in all the previous epochs as well as on the prior information on the dependence of the score and cost on the hyperparameters. The number of training epochs $\tau$ also depends on the history of model scores and computational costs. At the end of each epoch a decision is made whether to continue training or accept current model. Hence $\tau$ is not deterministically chosen before the learning starts and depends not only on the total computational budget but also on observed scores and costs. This may and does lead \algName to stop in some problems very early when the model happens to be so good that the expenditure on more training is not worthwhile (according to the criterion \eqref{eqn:2} which looks at the trade-off between score and cost).

We approximate \score and \cost as functions of hyperparameters using linear combinations of basis functions. We assume that the prior knowledge about the distribution of the coefficients of these linear combinations is multivariate Gaussian and is updated after each epoch using Kalman filtering. This update not only changes the mean vector but also adjusts the covariance matrix which represents the uncertainty about the coefficients and, consequently, about the score and cost mappings. Updated quantities feed back into the algorithm which decides whether to continue or stop training, and, in the former case, which hyperparameters to choose for the next epoch. We show that the updated distributions of coefficients are sufficient statistics of past observations of model scores and costs for the purpose of optimisation of the objective function \eqref{eqn:2}.

Our framework requires that a \emph{value function}, representing the optimum value of the objective function \eqref{eqn:2}, is available. It depends on the dimension of the problem and the trade-off coefficient $\gamma$ only, and is \score and \cost agnostic (in the sense that it does not depend on the particular problem and the choice of score and cost). This allows for an efficient offline precomputation and recycling of the value function maps. Users are able to share and reuse the same value function maps for different problems as long as the dimensionality of the hyperparameter space agrees. We demonstrate it in the validation section of this paper.

Similar to other AutoML frameworks \citep{feurer2015initializing, Vanschoren2019}, \algName natively enables meta-learning. Prior distributions for coefficients of basis functions determining the \score and \cost maps can be combined with the meta-data describing already completed learning tasks in a publicly available repositories. These data can then be used to select a more informative prior for a given problem and therefore, to \emph{warm-start} the training. In turn, this leads to reductions in the computational cost and improvements in the score.

The focus on optimising directly the trade-off between the model score and the total computational costs comes from two directions. Firstly, given the recent emphasis on the eco-efficiency of AI \citep{strubell2019energy, schwartz2019green}, \algName framework provides effective ways of weighing model quality and the employed computational resources as well as recycling information and, in turn, reducing computational resources required to train sufficiently good models. Secondly, with the democratisation of data analytics an increasing amount of data exploration will take place where users need to obtain useful (but not necessarily optimal) results within seconds or minutes \citep{demvsar2004orange, Holzinger2014}.

In summary, this paper's contributions are multifaceted. On the practical level, we develop an algorithm that allows optimal selection of hyperparameters for training and evaluation of models with an explicit trade-off between the model quality and the computational resources. From the eco-efficiency perspective our framework not only discourages unnecessarily resource intensive training but also naturally enables  transfer and accumulation of knowledge. Lastly, on the side of Markov decision processes, we solve a non-standard, mixed stochastic control and stopping problem for a partially observable dynamical system and design an efficient numerical scheme for the computation of the value function and optimal controls.

\subsection{Related literature}

To maintain the appeal of ML to its huge and newly acquired audience and to enable its use in systems that do not require human intervention, a large effort has been put in the development of algorithms and systems that allow automatic selection of an optimal ML method and an optimal tuning of its parameters. This line of research has largely operated under the umbrella name of AutoML (see \citep{autoMLbook2019}), standing for Automated Machine Learning, and boosted by a series of data crunching challenges under the same title (see, e.g., \citep{guyon2019analysis}). There are widely available AutoML packages in most statistical and data analytics software, see, e.g., Auto-Weka \citep{kotthoff2017auto}, Auto-sklearn \citep{feurer2015efficient}, as well as in all major commercial cloud computing environments.

The core of parameter optimisation typically takes a Bayesian approach in the form of alternating between learning/updating a model that describes the dependence of the \score on the hyperparameters and using this model and an acquisition function to select the next candidate hyperparameters. The different algorithms used to support the alternating model fitting and updating of parameters also affect its performance and suitability for specific problems. The three predominant algorithms are (for a discussion see \citep{eggensperger2013towards}): Gaussian Processes \citep{snoek2012practical}, random forest based Sequential Model-based Algorithm Configuration (SMAC) \citep{hutter2011sequential}, and Tree Parzen estimator \citep{bergstra2011algorithms}.

Our approach borrows from the ideas underlying Sequential Model-based Optimisation (SMBO) \citep{hutter2011sequential}. The surrogate model in our case is the basis-functions based representation of the score and cost maps. However, unlike SMBO, the choice of hyperparameters is not only based on the surrogate model but also on the expected future performance encoded in the value function. 

In all the above, time is certainly an important factor. It is however mostly treated either implicitly, through exploration-exploitation trade-off or through hard bounds, which can sometimes result in the algorithm providing no answer within the allocated time. A  more sophisticated optimisation of computational resources has recently started to draw interest. In a recent work \cite{falkner2018bohb}, the authors combine Bayesian optimisation to build a model and suggest potential new parameter configuration with Hyperband, which dynamically allocates time budgets to these new configurations applying successive halving \citep{li2017hyperband} to reduce time allocated to unsuccessful choices. \citep{yang2019oboe} uses meta-learning and imposes time constraints on the predicted running time of models. \cite{swersky2014freeze} discusses a freeze-thaw technique for (at least temporarily) pausing training for unpromising hyperparameters. \cite{Chandra2017} uses learning curves to stop training under-performing neural network models. Another related line of research is concerned with prediction of running time of algorithms (see \cite{Huang2010}, \cite{HUTTER2014} and many others). Those algorithms are akin to meta-learning, i.e., running time is predicted based on metafeatures of algorithms and their inputs.

Our approach differs from the above ideas in that it explicitely accounts for the trade-off between the model quality (score) and the cumulative cost of training. To the best of our knowledge this is the first attempt to bring the cost of computational resources to the same level of importance as the model score and include it in the objective function. Our developments are, however, parallel to some of the papers discussed above. For example, our approach would benefit from employment of learning curves to allow for incomplete training for hyperparameters which appear to produce suboptimal models. Our algorithm could complement OBOE \cite{yang2019oboe} by fine-tuning hyper-parameters of selected models. Such extensions are beyond the scope of this paper and will be explored in further research.

Our numerical method for computation of the value function complements a large strand of literature on regression Monte Carlo methods. The widest studied problem is that of optimal stopping (aka American option pricing) initiated by \cite{Tsitsiklis2001, Longstaff2001} and studied continuously to this day \cite{nadetal17}. Stochastic control problems were first addressed in \cite{Kharroubi2014} using control randomisation techniques and later studied in \cite{Balata2017} using regress later approach and in \cite{bachouch2018deep} employing deep neural networks. Numerical methods designed in our paper are closest in their spirit to regress later ideas of \cite{Balata2017} (see also \cite{bachouch2018deep}).

\subsection{Paper structure}

The paper is structured as follows. Section \ref{subsec:overview} gives details of the optimisation problem. Technical aspects of embedding this problem within the theory of Markov decision processes with partial information comes in the following subsection. Due to the noisy observation of the realised model score and cost (partial information), the resulting model is non-Markovian. Section \ref{sec:separation} reformulates the optimisation problem as a classical Markov decision model with a different state space. The rest of Section \ref{sec:stoch_control} is devoted to showing that the latter optimisation problem has a solution, developing a dynamic programming equation and showing how to compute optimal controls. Section \ref{sec:numerics} contains details of numerical computation of the value function and provides necessary algorithms. Validation of our approach is performed in Section \ref{sec:validation} which opens with a synthetic example to present, in a simple setting, properties of our solution. It is followed by 8 examples ecompassing a variety of datasets and modelling objectives. Appendices contain further details. Appendix \ref{app:dynamics} derives the dynamics of the Markov decision process of Section \ref{sec:separation}. Appendix \ref{app:aux} contains auxiliary estimates used in proofs of main results collected in Appendix \ref{app:proofs}. Detailed settings for all studied models are in Appendix \ref{app:D} while architectures of neural networks used in examples are provided in Appendix \ref{app:arch}.

\section{Stochastic control problem}\label{sec:stoch_control}

In this section, we embed the optimisation problem in the theory of stochastic control of partially observable dynamical systems \citep[Chapter 10]{Bertsekas1978} and lay foundations for its solution. 

\subsection{Model overview}\label{subsec:overview}
We start by providing a sketch of our framework. Recall the optimisation problem \eqref{eqn:2}. The random quantities $h_n$ and $t_n$ are not known as they depend on a dataset, a particular problem, and software and hardware used. We therefore embed the optimisation problem in a Bayesian learning setting. To this end, we assume that
\begin{equation}\label{eqn:h_t_n}
h_{n+1} = H(u_n) + \sigma_H(u_n) \epsilon_{n+1}, \qquad t_{n+1} = T(u_n) + \sigma_T(u_n) \eta_{n+1},
\end{equation}
where $(\epsilon_n)$ and $(\eta_n)$ are independent sequences of independent standard normal $N(0,1)$ random variables. The terms $\sigma_H(u_n) \epsilon_{n+1}$ and $\sigma_T(u_n) \eta_{n+1}$ correspond to the aforementioned random fluctuations of observed quantities around $H(u_n)$ and $T(u_n)$ and are Gaussian with the mean $0$ and the known variances $\sigma^2_H(u_n)$ and $\sigma_T^2(u_n)$. The quantities $H(u_n)$ and $T(u_n)$ are the \emph{true expected model score} and \emph{true expected model cost} given the choice of hyperparameters; for readability, we often omit the word \emph{expected}.

In order to embed the learning process in a computationally efficient Bayesian setting, we assume that $H(u)$ and $T(u)$ can be represented as linear combinations of basis functions:
\begin{equation}\label{eqn:HT}
H(u) = \sum_{j=1}^J \alpha_j \phi_j(u), \qquad T(u) = \sum_{k=1}^K \beta_k \psi_k(u),
\end{equation}
where functions $(\phi_j)_{j=1}^J$ and $(\psi_k)_{k=1}^K$ act from $\U$ to $\er$, where $\U$ is the set of admissible controls (hyperparameter values). The coefficients $\balpha = (\alpha_1, \ldots, \alpha_J)^T$ and $\bbeta = (\beta_1, \ldots, \beta_K)^T$ evolve while learning takes place. A priori, before learning starts, $\balpha$ follows a normal distribution with the mean $\mu^\alpha_0$ and the covariance matrix $\Sigma^\alpha_0$; the vector of coefficients $\bbeta$ follows a normal distribution with the mean $\mu^\beta_0$ and the covariance matrix $\Sigma^\beta_0$. These distributions are referred to as the prior distributions. We assume that a~priori $\balpha$ and $\bbeta$ are independent from each other. This choice of distributions is motivated by computational tractability but comes at the cost of modelling inaccuracy. In particular, we cannot guarantee that $T(u) \ge 0$ for any $u \in \U$ and any realisation of $\bbeta$. The consequences will be explored in more depth later in the paper.

\begin{algorithm}[tb]
\vspace{5pt}
\caption{Sketch of \algName}\label{alg:sketch}
\SetAlgoLined
\KwInit{$n := 0$}
\KwDo{continuation is optimal}{
    choose $u_n \in \U$\;
    train and validate a model with hyperparameters $u_n$ and store the score $h_{n+1}$ and the \runningtime $t_{n+1}$\;
    based on $h_{n+1}$ update the distribution of $\balpha$: new parameters $\mu^\alpha_{n+1}$ and $\Sigma^\alpha_{n+1}$\;
    based on $t_{n+1}$ update the distribution of $\bbeta$: new parameters $\mu^\beta_{n+1}$ and $\Sigma^\beta_{n+1}$\;
    $n \gets n+1$
}
$u^* \gets u_{n-1}$\;
\KwReturn{$u^*, \ldots$}
\vspace{5pt}
\end{algorithm}

Algorithm \ref{alg:sketch} provides a sketch of our approach. Given all information at \epoch $0$, i.e., given the prior distribution of $\balpha$ and $\bbeta$, we choose values of hyperparameters $u_0$ and then train and validate a model. Based on the observed model score $h_1$ and \runningtime $t_1$, we update the distribution of $\balpha$ and $\bbeta$ (the \emph{learning step}). Following this we check whether to continue learning and training. If the decision is to stop, we return the latest values of metaparemeters and further information, for example, the latest distributions of $\balpha$ and $\bbeta$ or the latest model. Otherwise, we repeat choosing hyperparameters, training, validation and update procedures. We continue until the decision is made to stop.

\begin{remark}
 In our Bayesian learning framework, the expectation in \eqref{eqn:2} is not only with respect to the randomness of observation errors $(\epsilon_n)$ and $(\eta_n)$ but also with respect to the prior distribution of $\balpha$ and $\bbeta$. Therefore, all possible realisations of the score and cost mappings are evaluated with more probable outcomes (as directed by the prior distribution) receiving proportionally more contribution to the value of the objective function.
\end{remark}

In the remaining of this subsection, we derive the optimisation problem that will be solved in this paper. This will result in modification to both terms in the objective function in \eqref{eqn:2}. 

\textbf{The score.} In Problem \eqref{eqn:2} the optimised quantity is the observed score of the most recently trained model. When the emphasis is on the selection of hyperparameters, one would replace $h_{\tau}$ by $H(u_{\tau-1})$ which corresponds to maximisation of the true expected score without the observation error:
\begin{equation}\label{eqn:2_true}
\sup_{(u_n), \tau} \ee \Big[H(u_{\tau-1}) - \gamma \sum_{n=1}^{\tau} t_n\Big].
\end{equation}

\textbf{The cost.} The optimal value for the problem \eqref{eqn:2_true}, as well as \eqref{eqn:2}, is infinite and offers controls $(u_n)$ of little practical use. This is because $\bbeta$ follows a normal distribution for the sake of numerical tractability. Indeed, one cannot find assumptions on the vector of basis functions $\bpsi$ to ensure that $\sum_{k=1}^K \beta_k \psi_k(u) \ge 0$ for all $u\in \U$ with probability one. Conversely, for any $\delta > 0$ we have
\[
\prob \Big(\sum_{k=1}^K \beta_k \psi_k(u) < -\delta \text{ for some $u \in \U$}\Big) > 0. 
\]
While learning one could identify, with increasing probability as $n$ increases, those realisations of $\bbeta$ for which there is $u \in \U$ with $\sum_{k=1}^K \beta_k \psi_k(u) < -\delta$ for some fixed $\delta > 0$. Then, continuing infinitely long with those controls $u$ would lead to a positive infinite value of the expression inside expectation.

In the view of the above remark, we truncate the cost to stay positive. The final form of the optimisation problem is
\begin{equation}\label{eqn:3}
\sup_{(u_n), \tau} \ee \Big[H(u_{\tau-1}) - \gamma \sum_{n=1}^\tau (t_n)^+ \Big],
\end{equation}
where $(t_n)^+ := \max(t_n, 0)$. We show in Lemma \ref{lem:upper_bound} that the optimal value of this problem is finite.

An alternative approach could be to amend the definition of $t_n$ to ensure non-negativity, but this would invalidate assumptions of the Kalman filter used in learning and, hence, make the solution of the problem numerically infeasible.

\subsection{Technical details}

Let $(\Omega, \ef, \prob)$ be the underlying probability space supporting the sequences $(\epsilon_n)_{n \ge 1}$ and $(\eta_n)_{n \ge 1}$ and the random variables $\balpha$ and $\bbeta$. By $\ef_n$ we denote the \emph{observable history} up to and including \epoch $n$, i.e., $\ef_n$ is the $\sigma$-algebra
\[
\ef_n = \sigma ( h_1, \ldots, h_{n}, t_1, \ldots, t_{n} ), \quad n \ge 1,
\]
with $\ef_0 = \{\emptyset, \Omega\}$. The choice of hyperparameters $u_n$ must be based only on the observable history, i.e., $u_n$ must be $\ef_n$-measurable for any $n \ge 0$. The variable $\tau$ which governs the termination of training must be an $(\ef_n)$-stopping time, i.e., $\{\tau=n\} \in \ef_n$ for $n \ge 0$ -- the decision whether to finish training is based on the past and present observations only. The difficulty lies in the fact that observations combine evaluations of the true expected score function and the running time function with errors, c.f., \eqref{eqn:h_t_n}. This places us in the framework of stochastic control with partial observation \citep[Chapter 10]{Bertsekas1978}.

Denote $\bphi(u) = (\phi_1(u), \ldots, \phi_J(u))^T$ and $\bpsi(u) = (\psi_1(u), \ldots, \psi_K(u))^T$. Equations \eqref{eqn:h_t_n} can be written in a vector notation as
\begin{equation}\label{eqn:vec_notation}
h_{n+1} = \balpha^T \bphi(u_n) + \sigma_H(u_n) \epsilon_{n+1}, \qquad t_{n+1} = \bbeta^T \bpsi(u_n) + \sigma_T(u_n) \eta_{n+1}.
\end{equation}
We will use this notation throughout the paper.

The power of the theory of Markov decision processes is in its tractability: an optimal control can be written in terms of the current state of the system. However, when the system is not fully observed, there is an obvious benefit to take all past observations into account to deterimine possible values of unobservable elements of the system. Indeed, due to the form of observation \eqref{eqn:h_t_n}-\eqref{eqn:HT}, one can infer much more about $\balpha$ and $\bbeta$ when taking into account all available past readings of model scores and training costs. It turns out that optimal control problems of the form studied in this paper can be rewritten into equivalent systems with a bigger state space but with full observation for which classical dynamic programming can be applied.

Before we proceeed with this programme, we state standing assumptions:
\begin{assumption}
\item[(A)] Basis functions $(\phi_j)_{j=1}^J$ and $(\psi_k)_{k=1}^K$ are bounded on $\U$. \label{ass:bounded}
\item[(S)] Observation error variances are bounded and separated from zero, i.e. \label{ass:sigma_u}
\[
\inf_{u \in \U} \sigma_H(u) > 0, \quad \inf_{u \in \U} \sigma_T(u) > 0, \quad \sup_{u \in \U} \sigma_H(u) + \sigma_T(u) < \infty.
\]
\end{assumption}
It is not required for the validity of mathematical results that basis functions are orthogonal in a specific sense or even linearly independent. However, linear independence is required for identification purposes and it speeds up learning.

\subsection{Separation principle}\label{sec:separation}
Denote by $\cA_n$ the conditional distribution of $\balpha$ given $\ef_n$. It follows from an application of the discrete-time Kalman filter \cite[Section 4.7]{Bensoussan2018} that $\cA_n$ is Gaussian with the mean $\mu^\alpha_n$ and the covariance matrix $\Sigma^\alpha_n$ given by the following recursive formulas:
\begin{equation}\label{eqn:filter_a}
 \begin{aligned}
(\Sigma^\alpha_{n+1})^{-1} &= (\Sigma^\alpha_n)^{-1} + \frac{1}{\sigma_H(u_n)^2} \bphi(u_n) \bphi(u_n)^T\\
\mu^\alpha_{n+1} &= \mu^\alpha_n + \frac{1}{\sigma_H(u_n)^2} \Sigma^\alpha_{n+1} \bphi(u_n) \big( h_{n+1} - \bphi(u_n)^T \mu^\alpha_n \big).
 \end{aligned}
\end{equation}

Denote by $\cB_n$ the conditional distribution of $\bbeta$ given $\ef_n$. By the same arguments as above, it is also Gaussian with the mean $\mu^\beta_n$ and the covariance matrix $\Sigma^\beta_n$ given by the following recursive formulas:
\begin{equation}\label{eqn:filter_b}
 \begin{aligned}
(\Sigma^\beta_{n+1})^{-1} &= (\Sigma^\beta_n)^{-1} + \frac{1}{\sigma_T(u_n)^2} \bpsi(u_n) \bpsi(u_n)^T\\
\mu^\beta_{n+1} &= \mu^\beta_n + \frac{1}{\sigma_T(u_n)^2} \Sigma^\beta_{n+1} \bpsi(u_n) \big( t_{n+1} - \bpsi(u_n)^T \mu^\beta_n  \big).
 \end{aligned}
\end{equation}
Furthermore, it follows from the definition of $h_{n+1}$ and $t_{n+1}$ that the conditional distribution of $h_{n+1}$ given $\ef_n$ and control $u_n$ is Gaussian
\begin{equation*}
 N\Big((\mu^\alpha_{n})^T \bphi(u_n),\ \bphi(u_n)^T \Sigma^\alpha_{n} \bphi(u_n) + \sigma^2_H(u_n) \Big),
\end{equation*}
and the conditional distribution of $t_{n+1}$ given $\ef_n$ and control $u_n$ is also Gaussian
\begin{equation*}
N\Big((\mu^\beta_{n})^T \bpsi(u_n),\ \bpsi(u_n)^T \Sigma^\beta_{n} \bpsi(u_n) + \sigma^2_T(u_n) \Big).
\end{equation*}

Measure-valued stochastic processes $(\cA_n)$, $(\cB_n)$ are adapted to filtration $(\ef_n)$. This would have been of little practical use for us if it were not for the fact that those measure valued processes take values in a space of Gaussian distributions, so the state is determined by the mean vector and the covariance matrix. It is then clear that the process $(\mu^\alpha_n, \Sigma^\alpha_n, \mu^\beta_n, \Sigma^\beta_n, t_n)$ is a Markov process (we omit $h_n$ as it does not appear in the objective function \eqref{eqn:3}). Its dynamics are explicitly given in Appendix \ref{app:dynamics}.

The following theorem shows that the optimisation problem can be restated in terms of these variables instead of $\balpha$ and $\bbeta$, i.e., in terms of observable quantities only. This reformulates the problem as a fully observed Markov decision problem which can be solved using dynamic programming methods.
\begin{theorem}\label{thm:separation}
Optimisation problem \eqref{eqn:3} is equivalent to
\begin{equation}\label{eqn:4}
\sup_{(u_n), \tau} \ee \Big[(\mu^\alpha_\tau)^T \bphi(u_{\tau-1}) - \gamma \sum_{n=1}^\tau (t_n)^+ \Big].
\end{equation}
\end{theorem}
The proof of this theorem as well as other results are collected in Appendix \ref{app:proofs}.

\subsection{Dynamic programming formulation}

By inspecting the dynamics of $(\mu^\alpha_n, \Sigma^\alpha_n, \mu^\beta_n, \Sigma^\beta_n, t_n)$ in Appendix \ref{app:dynamics} one notices that the dependence of the state at time $n+1$ on time $n$ is only via the parameters of the filters $(\cA_n)$, $(\cB_n)$ which contain sufficient statistic about $(h_i)_{i=1}^n$ and $(t_i)_{i=1}^n$. Furthermore, the dynamics is time-homogeneous, i.e., it does not depend on the epoch $n$. Hence, the value function of the optimisation problem \eqref{eqn:4} depends only on 4 parameters:
\begin{equation}\label{eqn:7}
\begin{aligned}
&V(\mu^\alpha, \Sigma^\alpha, \mu^\beta, \Sigma^\beta)\\
&= \sup_{(u_{n}), \tau \ge 1} \ee\Big[(\mu^\alpha_\tau)^T \bphi(u_{\tau-1}) - \gamma \sum_{n=1}^\tau (t_n)^+ \Big| \mu^\alpha_0 = \mu^\alpha_{\phantom{0}}, \Sigma^\alpha_0 = \Sigma^\alpha_{\phantom{0}}, \mu^\beta_0 = \mu^\beta_{\phantom{0}}, \Sigma^\beta_0 = \Sigma^\beta_{\phantom{0}}\Big].
\end{aligned}
\end{equation}
The function $V$ gives an optimal value of the optimisation problem given the distributions of $\balpha$ and $\bbeta$ are Gaussian with the given parameters. The expression on the right-hand side is well defined and the value function does not admit values $\pm \infty$, see Lemma \ref{lem:upper_bound}.

In preparation for stating the dynamic programming principle, define an operator acting on measurable functions $\varphi(\mu^\alpha, \Sigma^\alpha, \mu^\beta, \Sigma^\beta)$ as follows
\begin{equation}\label{eqn:tau_a}
\cT \varphi (\mu^\alpha, \Sigma^\alpha, \mu^\beta, \Sigma^\beta) =
\sup_{u \in \U} \ee \Big[ - \gamma (t_1)^+ + \max \big( (\mu^\alpha_1)^T \bphi(u), \varphi(\mu^\alpha_1, \Sigma^\alpha_1, \mu^\beta_1, \Sigma^\beta_1) \big)\Big| \mu^\alpha, \Sigma^\alpha, \mu^\beta, \Sigma^\beta\Big],
\end{equation}
where the expectation is with respect to $\mu^\alpha_1, \Sigma^\alpha_1, \mu^\beta_1, \Sigma^\beta_1, t_1$ conditional on $\mu^\alpha_0 = \mu^\alpha, \Sigma^\alpha_0 = \Sigma^\alpha_{\phantom{0}}, \mu^\beta_0 = \mu^\beta_{\phantom{0}}, \Sigma_0^\beta = \Sigma^\beta_{\phantom{0}}$; we will use this shorthand notation whenever it does not lead to ambiguity. The expectation of the first term containing only $t_1$ can be computed in a closed form improving efficiency of numerical calculations:
\begin{multline}\label{eqn:reform}
\cT \varphi (\mu^\alpha, \Sigma^\alpha, \mu^\beta, \Sigma^\beta) =
\sup_{u \in \U} \Big[ - \gamma \Tcost\Big((\mu^\beta)^T \bpsi(u),\, \bpsi(u)^T \Sigma^\beta \bpsi(u) + \sigma^2_T(u)\Big)\\
+ \ee \Big[\max \Big( (\mu^\alpha_1)^T \bphi(u),\, \varphi( \mu^\alpha_1, \Sigma^\alpha_1, \mu^\beta_1, \Sigma^\beta_1) \Big)\Big| \mu^\alpha, \Sigma^\alpha, \mu^\beta, \Sigma^\beta\Big]\Big],
\end{multline}
where
\[
\Tcost (m, s^2) =\frac{s}{\sqrt{2\pi}} e^{-\frac{m^2}{2s^2}} + m \Phi \Big(\frac{m}{s}\Big), 
\]
and $\Phi$ is the cumulative distribution function of the standard normal distribution. The justification of this formula is in Appendix \ref{app:proofs}.

In order to approximate $V$, consider an iterative scheme:
\begin{equation}\label{eqn:iter_scheme}
\begin{aligned}
V_1 (\mu^\alpha, \Sigma^\alpha, \mu^\beta, \Sigma^\beta) &= \sup_{u \in \U} \ee\Big[ - \gamma (t_1)^+ + (\mu^\alpha_1)^T \bphi(u) \Big|\mu^\alpha, \Sigma^\alpha, \mu^\beta, \Sigma^\beta \Big]\\
V_{N+1}(\cdot) &= \cT V_N (\cdot), \qquad N \ge 1.
\end{aligned}
\end{equation}
The following theorem guarantees that the above scheme is well defined, i.e., $V_N$ does not take values $\pm \infty$ and the expression under expectation in $\cT V_N$ is integrable, so that the operator $\cT$ can be applied to $V_N$. The proof is in Appendix \ref{app:proofs}.

\begin{theorem}\label{thm:2}
The following statements hold true.\nopagebreak
\begin{enumerate}
\item The iterative scheme \eqref{eqn:iter_scheme} is well defined.
\item  $V_N$ is the value function of the problem with at most $N$ training epochs, i.e.
\begin{align*}
&V_N(\mu^\alpha, \Sigma^\alpha, \mu^\beta, \Sigma^\beta)
= \sup_{(u_{n})_{n \ge 0},\, 1 \le \tau \le N} \ee\Big[(\mu^\alpha_{\tau})^T \bphi(u_{\tau-1}) - \gamma \sum_{n=1}^\tau (t_n)^+ \Big| \mu^\alpha,
\Sigma^\alpha, \mu^\beta, \Sigma^\beta\Big].
\end{align*}
\item $V = \lim_{N \to \infty} V_N$ and the sequence of functions $(V_N)$ is non-decreasing in $N$. Furthermore, $|V| < \infty$.
\item The value function $V$ defined in \eqref{eqn:7} is a fixed point of the operator $\cT$, i.e., satisfies the \emph{dynamic programming equation}
\begin{multline}\label{eqn:dpp_2}
V( \mu^\alpha,  \Sigma^\alpha,  \mu^\beta,  \Sigma^\beta)\\ 
= \sup_{u \in \U} \ee \Big[ - \gamma ( t_1)^+ 
+ \max \big( ( \mu^\alpha_1)^T \bphi(u), V( \mu^\alpha_1,  \Sigma^\alpha_1,  \mu^\beta_1,  \Sigma^\beta_1) \big)\Big|\mu^\alpha, \Sigma^\alpha, \mu^\beta, \Sigma^\beta  \Big].
\end{multline}
\end{enumerate}
\end{theorem}

Assertion 4 is of theoretical interest: it is the statement of the dynamic programming equation for the infinite horizon stochastic control problem. Assertion 2 states that $V_N$ is the value function for the problem with at most $N$ training epochs, which suggests what trade-off is made when using $V_N$ instead of $V$. Furthermore, $V_N \le V$ by Assertion 3, which, as we will see in the next section, may affect the stopping. We will refer to those results in the numerical section of the paper.

\subsection{Identifying control and stopping}\label{subsec:whentostop}
If $V$ was available, then the choice of control $u_n$ and the decision about stopping in Algorithm \ref{alg:sketch} would follow from classical results in optimal control theory. Indeed, one would choose
\begin{multline}\label{eqn:relaxed_u}
u_n = \argmax_{u \in \U} \ee \Big[ - \gamma ( t_{n+1})^+ 
+ \max \big( ( \mu^\alpha_{n+1})^T \bphi(u), V( \mu^\alpha_{n+1},  \Sigma^\alpha_{n+1},  \mu^\beta_{n+1},  \Sigma^\beta_{n+1}) \big)\\
\Big|\mu^\alpha_n, \Sigma^\alpha_n, \mu^\beta_n, \Sigma^\beta_n, u_n = u  \Big],
\end{multline}
and then train and validate the model using hyperparameters $u_n$, compute $\mu^\alpha_{n+1},  \Sigma^\alpha_{n+1},  \mu^\beta_{n+1},  \Sigma^\beta_{n+1}$ via \eqref{eqn:filter_a}-\eqref{eqn:filter_b}, and stop when
\begin{equation}\label{eqn:relaxed_stop}
(\mu^\alpha_{n+1})^T \bphi(u_n) \ge V( \mu^\alpha_{n+1},  \Sigma^\alpha_{n+1},  \mu^\beta_{n+1},  \Sigma^\beta_{n+1}).
\end{equation}

A natural adaptation of this procedure to the case when one only knows $V_N$ is to replace $V$ by $V_N$ in both places in the above formulas. 
Another possibility is to follow an optimal strategy corresponding to the computed value function, i.e, with the given maximum number of training epochs. If one knows $V_1, \ldots, V_N$, then for $n = 0, \ldots, N-1$,
\begin{multline*}
u_n = \argmax_{u \in \U} \ee \Big[ - \gamma ( t_{n+1})^+ 
+ \max \big( ( \mu^\alpha_{n+1})^T \bphi(u), V_{N-n}( \mu^\alpha_{n+1},  \Sigma^\alpha_{n+1},  \mu^\beta_{n+1},  \Sigma^\beta_{n+1}) \big)\\\Big|\mu^\alpha_n, \Sigma^\alpha_n, \mu^\beta_n, \Sigma^\beta_n, u_n = u  \Big],
\end{multline*}
and stop when
\[
n = N \quad \text{or} \quad (\mu^\alpha_{n+1})^T \bphi(u_n) \ge V_{N-n}( \mu^\alpha_{n+1},  \Sigma^\alpha_{n+1},  \mu^\beta_{n+1},  \Sigma^\beta_{n+1}),
\]
where, with an abuse of notation, we put $V_0 \equiv -\infty$, i.e., at time $n=N$ we always stop and we take that into account in choosing $u_N$ (c.f. the definition of $V_1$).

Let us refer to the first approach as the \emph{relaxed method} and the second approach as the \emph{exact method}. The advantage of the exact method is that it has strong underlying mathematical guarantees but it is limited to $N$ training epochs with $N$ chosen a priori. The relaxed method does not impose, a priori, such constraints but it may stop too early as the right-hand side of the condition \eqref{eqn:relaxed_stop} is replaced by $V_N \le V$. 

\section{Numerical implementation of \algName}\label{sec:numerics}

In this section, we focus on describing the methodology behind our implementation of \algName. We show how to numerically approximate the iterative scheme \eqref{eqn:iter_scheme} and how to explicitly account for numerical errors when using estimated $(V_n)_{n \ge 1}$ for the choice of control (hyperparameters) $u_n$ and for the stopping condition in Algorithm \ref{alg:sketch}. We wish to reiterate that, apart from the on-the-fly method (Section \ref{subsec:onthefly}), a numerical approximation of the value functions $V_n$, $n \ge 1$, needs to be computed only once and then applied in Algorithm \ref{alg:core} for various datasets and modelling tasks. We demonstrate this in Section \ref{sec:validation}.

\subsection{Notation}
Denote by $\X$ the state space of the process $x_n = (\mu^\alpha_n, \Sigma^\alpha_n, \mu^\beta_n, \Sigma^\beta_n)$. For $x = (\mu^\alpha, \Sigma^\alpha, \mu^\beta, \Sigma^\beta) \in \X$, define
\begin{align*}
    m^{\alpha}(u, x) = (\mu^{\alpha})^T \bphi(u), \quad & \quad s^{\alpha}(u, x)^2 = \bphi(u)^T \Sigma^{\alpha}\bphi(u) + \sigma_H^2(u),\\
     m^{\beta}(u, x) = (\mu^{\beta})^T \bpsi(u), \quad & \quad s^{\beta}(u, x)^2 = \bpsi(u)^T \Sigma^{\beta}\bpsi(u) + \sigma_T^2(u).
\end{align*}
Given $h, t \in \er$ and $x \in \X$, we write $\Kalman(x, h, t)$ for the output of the Kalman filter \eqref{eqn:filter_a}-\eqref{eqn:filter_b} (with an obvious adjustment to the notation: $x$ is taken to be the state at time $n$ and $h, t$ are the observation of the score $h_{n+1}$ and the cost $t_{n+1}$).

For convenience of implementation we assume that $\U = [0,1]^p$ which can be achieved by appropriate normalisation of the control set. We show in Section \ref{sec:validation} that this assumption is not detrimental even if some of hyperparameters take discrete values.

\subsection{Computation of the value function}
We compute the value function via the iterative scheme \eqref{eqn:iter_scheme}. We use ideas of regression Monte Carlo for controlled Markov processes, particularly the regress-later method \cite{Balata2017, Hure2018}, and approximate the value function $V_n$, $n=1, \ldots, N$, learnt over an appropriate choice of states $\cloud \subset \X$. The choice of the \emph{cloud} $\cloud$ is crucial for the accuracy of the approximation of the value function at points of interest; it is discussed in Section \ref{subsec:cloud}. We fix a grid of controls $\grid$ for computing an approximation of the Q-value as in practical applications we found it to perform better than a (theoretically supported) random choice of points in the control space $\U$.

\begin{algorithm}[tp]
\caption{$\regress(x, \varphi; \grid, N_s)$}\label{alg:lambda}
\vspace{5pt}
\SetAlgoLined
\KwIn{$x$, $\varphi$, $\grid$, $N_s$}
\KwResult{$\mathcal{L}_p$}
\BlankLine
        \For{$u^i \in \grid$}{
            Sample $(h^k)_{k=1}^{N_s} \sim N(m^{\alpha}(u^i, x), s^{\alpha}(u^i, x)^2)$ \;
            Sample $(t^k)_{k=1}^{N_s} \sim N(m^{\beta}(u^i, x), s^{\beta}(u^i, x)^2)$ \;
            $y^k \gets \Kalman(x, h^k, t^k)$ for $k=1, \ldots, N_s$\;
            $p^i \gets -\gamma\Upsilon(m^{\beta}(u^i, x), s^{\beta}(u^i, x)^2) + \frac{1}{N_s}\sum\limits_{k=1}^{N_s}\max \big(m^\alpha(u^i, y^k), \varphi(y^k)\big)$\;
        }
        $\mathcal{L}_p \gets \argmin_{\xi \in \mathcal{F}_p} \sum\limits_{i=1}^{|\grid|}l_p(p^i,\xi(u^i))$ \tcc*[r]{Regression}
\vspace{5pt}
\end{algorithm}
\begin{algorithm}[tp]
\caption{Value iteration for \algName (VI/Map)}\label{alg:vi}
\vspace{5pt}
\SetAlgoLined
\KwIn{$\varphi$; $\cloud$; $\grid$; $N$; $N_s$\tcc*[r]{Usually, set $\varphi\equiv -\infty$}}
\KwResult{$\wtl V_1, \ldots, \wtl V_N$}
\BlankLine
\KwInit{$n := 0$; $\wtl V_0 = \varphi$}
 \While{$n < N$}{
     \For{$x^j \in \cloud$}{
        $\mathcal{L}_p^j \gets \regress (x^j, \wtl V_n; \grid, N_s)$\\
        $q^j \gets \max_{u \in \mathcal{U}} \mathcal{L}_p^j(u)$
     }
     $\wtl V_{n+1} \gets \argmin_{\xi \in \mathcal{F}_{\varphi}} \sum\limits_{j=1}^{|\cloud|}l_{\varphi}(q^j,\xi(x^j))$ \tcc*[r]{Regression}
     $n \gets n + 1$
 }
\vspace{5pt}
\end{algorithm}

Computation of the value function is presented in two algorithms. Algorithm \ref{alg:lambda} finds an approximation of the expression inside optimisation in $\cT \varphi(x)$ (see \eqref{eqn:reform}), while Algorithm \ref{alg:vi} shows how to adapt value iteration approach to the framework of our stochastic optimisation problem. Details follow.

Algorithm \ref{alg:lambda} defines a function $\regress (x, \varphi; \grid, N_s)$ that computes an approximation to the mapping
\begin{equation}\label{eqn:q1}
u \mapsto - \gamma \Tcost\big(m^\beta(u, x),\, s^\beta(u, x)^2\big)\\
+ \ee \big[\max \big(m^\alpha(u, x_{1}),\, \varphi(x_1) \big) \big| x_{0} = x\big].
\end{equation}
It is done by evaluating the outcome of each control $u^i$ in $\grid$ (line 1 in Algorithm \ref{alg:lambda}) using Monte Carlo averaging over $N_s$ samples from the distribution of the state $x_{1}$ given $x_0 = x$ and the control $u^i$; this averaging is to approximate the expectation. However, the regression in line 7 employs further cross-sectional information to reduce error in Monte Carlo estimates of this expectation as long as the sampling of $(h_k)$ and $(t_k)$ is independent between different values $u^i$. This effect has mathematical grounding as a Monte Carlo approximation to ortogonal projections in an appropriate $L^2$ space if the loss function is quadratic: $l_p(p, z) = (p-z)^2$, c.f., \cite{Tsitsiklis2001, Balata2017}. In particular, it is valid to use $N_s=1$, but larger values may result in more efficient computation: one seeks a trade-off between the size of $\grid$ and the size of $N_s$ to optimise computational time and achieve sufficient accuracy. 

The space of functions $\ef_p$ over which the optimisation in regression in line 7 is performed depends on the choice of functional approximation. In regression Monte Carlo it is a linear space spanned by a finite number of basis functions. Other approximating families $\ef_p$ include neural networks, splines and regression Random Forest. Although we do not include it explicitely in the notation, the optimisation may further involve regularisation in the form of elastic nets \cite{Zou2005} or regularised neural networks, which often improves the real-world performance.

Algorithm \ref{alg:vi} iteratively applies the operator $\cT$ and computes regression approximation of the value function. The loop in lines 2-5, computes optimal control and the value for each point from the cloud $\cloud$. In line 3, we approximate the mapping
\[
u \mapsto - \gamma \Tcost\big(m^\beta(u, x^j),\, s^\beta(u, x^j)^2\big)\\
+ \ee \big[\max \big(m^\alpha(u, x_{n+1}),\, \wtl V_n( x_{n+1}) \big) \big| x_{n} = x^j\big],
\]
using Algorithm \ref{alg:lambda}. Regression in line 6 uses pairs: the state $x^j$ and the corresponding optimal value $q^j$ to fit an approximation $\wtl V_{n+1}$ of the value function $V_{n+1}$. This approximation is needed in the following iteration of the loop (in line 3) and as the output of the algorithm to be used in application of \algName to a particular data problem. The family of functions $\ef_\varphi$ (as $\ef_p$) includes linear combinations of basis functions (for regression Monte Carlo), neural networks or regression Random Forest. Importantly, the dimensionality of the state $x$ is much larger than of the control $u$ itself, so efficient approximation of the value function requires more intricate methods than in the approximation $\regress (x, \varphi; \grid, N_s)$ of the q-value. As in Algorithm \ref{alg:lambda}, some regularisation may benefit real-world performance. The loss function is commonly the squared error $l_\varphi (q, z) = (q-z)^2$ but others can also be used (particularly when using neural networks) at the cost of more difficult or of no justification of convergence of the whole numerical procedure.

\subsection{\algName}\label{sec:autoBML}

\begin{algorithm}[tb]
\vspace{5pt}
\caption{\algName}\label{alg:core}
\SetAlgoLined
\KwIn{$x_0, \epsilon$, $\wtl{V}$, $\grid'$, $N_s'$}
\KwResult{$h^*, u^*, T^*, x^*$}
\BlankLine
\KwInit{$n := -1$, $x = x_0$}
$\tilde\Lambda \gets \regress (x, (1-\epsilon) \wtl{V} ; \grid', N_s')$ \tcc*[r]{evaluate $\Lambda^\epsilon$ at $x$}
$u_0' \gets \arg\max_{u \in \U} \tilde\Lambda(u)$ \;
\KwDo{$ m^{\alpha}(u_{n}', x) < \tilde\Lambda (u_{n+1}') $}{
    $n \gets n+1$\;    
    $h_{n+1}, t_{n+1} \gets \cM(u_{n}')$ \tcc*[r]{run external ML and record h, t}
    $x \gets \Kalman(x, h_{n+1}, t_{n+1})$ \tcc*[r]{update beliefs}
    $\tilde\Lambda \gets \regress (x, (1-\epsilon) \wtl{V} ; \grid', N_s')$ \tcc*[r]{evaluate $\Lambda^\epsilon$ at $x$}
    $u_{n+1}' \gets \arg\max_{u \in \U} \tilde\Lambda(u)$ \tcc*[r]{determine optimal control}
}
$(h^*,u^*) \gets (m^{\alpha}(u_{n}', x), u'_n)$\;
$T^* \gets \sum_{k=1}^{n+1} t_k$ \;
\vspace{5pt}
\end{algorithm}

The core part of the \algName is given in Algorithm \ref{alg:core}. We use the notation $\cM$ to denote an external Machine Learning (ML) process that outputs \score $h$ and cost $t$ for the choice of hyperparameters $u\in\U$. The `raw' score $\hat{h}$ and the `raw' cost $\hat{t}$ obtained from the ML routine are transformed onto the interval $[0,1]$ with the user-specified map (robustness to a mis-specification of this map is discussed below):
\[
(\hat{h}, \hat{t}) \mapsto (h, t).
\]
For the cost $t$ we usually use an affine transformation based on a best guess of maximum and minimum values. For the score $h$ the choice depends on the problem at hand: in our practice we have employed affine and affine-log transformations. Mis-specification of these maps so that the actual transformed values are outside of the interval $[0,1]$ are accounted for by learning $\wtl{V}$ using a cloud $\cloud$ containing states that map onto distributions with resulting functions $H$ and $T$ not bounded by $[0,1]$ with non-negligible probability. This robustness does not, however, relieve the user from a sensible choice of the above transformations as significantly bad choices have a tangible effect on the behaviour of the algorithm.

We present an algorithm for the \emph{relaxed method} as introduced in Section \ref{subsec:whentostop}; an adaptation of Algorithm \ref{alg:core} to the \emph{exact method} is easy and has been omitted. 

The input of Algorithm \ref{alg:core} consists of an approximation $\wtl{V}$ of the value function $V$. An optimal control $u_{n}$ at epoch $n$ given state $x_n$ should be determined by maximisation of the expression \eqref{eqn:q1} over $u \in \U$ putting $\varphi = \wtl V$ and $x = x_n$. For each $u$, this expression compares two possibilities available after applying control $u$: either stopping or continuing training. These two options correspond to strategically different choices of $u$; the first one maximises score in the next epoch while the latter prioritises learning. The learning option is preferred if future rewards (as given by $\wtl{V}$) exceed present expected score. The approximation $\wtl{V}$ of $V$ may be burdened by errors which may result in an incorrect decision and drive the algorithm into an unnecessary `learning spree' (the opposite effect of premature stopping is less dangerous as it does not lead to unnecessarily long computation times). To mitigate this effect, we replace \eqref{eqn:q1} with
\begin{equation}\label{eq:Lambda}
    \Lambda^\epsilon(u; x, \epsilon, \wtl{V}) =  -\gamma\Upsilon(m^{\beta}(u, x),s^{\beta}(u, x)^2) + E\Big[\max \big( (m^\alpha(u, x_1), \wtl{V}(x_1)\left(1 - \epsilon\right)\big)\Big| x_0 = x\Big],
\end{equation}
for an error adjustment $\epsilon \ge 0$. Notice that $\epsilon = 0$ brings back the original formula from \eqref{eqn:q1}. The approximation of the mapping $u \mapsto \Lambda^\epsilon(u; x, \epsilon, \wtl{V})$ is computed by Algorithm \ref{alg:lambda} with $\varphi(x) = (1-\epsilon) \wtl{V} (x)$. The grid $\grid'$ and the number of Monte Carlo iteration $N_s'$ are usually taken much larger than in the value iteration algorithm since one needs only one evaluation at each epoch instead of 1000s of evaluations needed in Algorihm \ref{alg:vi}. Furthemore, more accurate computation of the mapping $\Lambda^\epsilon(\cdot; x, \epsilon, \wtl{V})$ was shown to improve real-life performance.

The stopping condition in line 9 is an approximation of the condition $m^{\alpha}(u_{n}, x) < V(x)$. Indeed, using that $V = \cT V$, we could write
$m^{\alpha}(u_{n}, x) < \cT V(x)$ and approximate the right-hand side by $\sup_{u \in \U} \Lambda^\epsilon (u; x, 0, \wtl{V})$ with $\epsilon=0$. Using non-zero $\epsilon$ 
is to counteract numerical and approximation errors as discussed above.

The stopping condition $m^{\alpha}(u_{n}, x) < V(x)$ could also be implemented as $m^{\alpha}(u_{n}, x) < (1-\epsilon) \wtl{V}(x)$. We do not do it for two reasons. Firstly, $\tilde\Lambda(u_{n+1}')$ is a more accurate approximation of the value function $V$ at point $x$ due to averaging effect of evaluating $\cT \wtl{V}$ at the required point $x$ and with computations done with much higher accuracy stemming from a finer grid $\grid'$ and a larger number of Monte Carlo iterations $N_s'$. Secondly, this brings the stopping condition in line with the choice of optimal control in line 8 which is also based on $\tilde\Lambda$ and comes at no extra cost. 

\subsection{States and clouds}\label{subsec:cloud}

Approximations $\wtl{V}_n$ in Algorithm \ref{alg:vi} are learnt over states in cloud $\cloud \subset \X$, and therefore the quality of these maps depends on the choice of $\cloud$. This cloud needs to comprise sufficiently many states encountered over the running of the model in order to provide good approximations for the choice of control $u_n$ and the stopping decision in 
Algorithm \ref{alg:core}. In the course of learning, the norm of the covariance matrices $\Sigma^\alpha_n$ and $\Sigma^\beta_n$ becomes smaller (the filter is concentrated ever more on the mean) asymptotically approaching a \emph{truth}:

\begin{definition}[Truth]
A state $x = (\mu^\alpha, \Sigma^\alpha, \mu^\beta, \Sigma^\beta) \in \X$ is called a \textit{truth} if $\Sigma^\alpha = 0$  and $\Sigma^\beta = 0$.
The set of all truths in $\X$ is denoted by $\X_{\infty}$.
\end{definition}

We argue that the learning cloud $\cloud$ should contain the following:
\begin{enumerate}[label=\alph*)]
\item an adequate number of states that are elements of $\X_{\infty}$,
\item states that are suitably close to priors $x_0$ one will use in the Algorithm \ref{alg:core},
\item an adequate number of states `lying between' the above.
\end{enumerate}

A failure to include states which are truths (point a) inhibits the ability for \algName algorithm to adequately stop. Conversely, a failure to include states which Algorithm \ref{alg:core} encounters in its running (point c) inhibits the ability to adequately continue. A failure to include a sufficient number of states that are close to priors $x_0$ in (b) impacts the coverage of the state space by states (point c), and, further, due to the errors in extrapolation, the decisions made in the first step of the algorithm.

We suggest the following approach for construction of the cloud $\cloud$. For given $N_c, K \ge 1$, we repeat $N_c$ times the steps: 
\begin{enumerate}[label=\arabic*)]
\item sample $\mu^\alpha$ and $\mu^\beta$ from appropriate distributions, e.g., multivariate Gaussian,
\item sample $\Sigma^\alpha$ and $\Sigma^\beta$ from appropriate covariance distributions, e.g. Wishart or LKJ \cite{lewandowski2009},
\item add states
\[
\Big(\mu^\alpha, \frac{k}{K} \Sigma^\alpha, \mu^\beta, \frac{k}{K} \Sigma^\beta\Big), \qquad k = 0, \ldots, K,
\]
to the cloud $\cloud$.
\end{enumerate}
This results in a cloud of size $N_c (K+1)$.

Note that if in steps 1 and 2 one uses distributions with the means equal to possible priors $x_0$, point (b) above is fulfilled. At step 3, $k=0$ gives zero covariance matrices, i.e.,  truths according to the definition above. It also highlights that distributions used in step 1 should be sufficiently dispersed to cover possible truths for models for which \algName will be used.

\subsection{On-the-fly}\label{subsec:onthefly}

\begin{algorithm}[tb]
\caption{$\OTF(x, N; \varphi, \grid, N_s)$}\label{alg:OTF}
\vspace{5pt}
\SetAlgoLined
\KwIn{$x$, $N$, $\varphi$, $\grid$, $N_s$, }
\KwResult{$\mathcal{L}_p$}
\BlankLine
        \For{$u^i \in \grid$}{
            Sample $(h^k)_{k=1}^{N_s} \sim N(m^{\alpha}(u^i, x), s^{\alpha}(u^i, x)^2)$ \;
            Sample $(t^k)_{k=1}^{N_s} \sim N(m^{\beta}(u^i, x), s^{\beta}(u^i, x)^2)$ \;
            $y^k \gets \Kalman(x, h^k, t^k)$ for $k=1, \ldots, N_s$\;
            \uIf{$N > 1$} 
            { $v \gets \OTF (y^k, N-1; \varphi, \grid, N_s)$\;
            $v^k \gets \sup_{u \in \U} v (u)$\;}
            \Else
            { $v^k \gets \varphi(y^k)$}
            $p^i \gets -\gamma\Upsilon(m^{\beta}(u^i, x), s^{\beta}(u^i, x)^2) + \frac{1}{N_s}\sum\limits_{k=1}^{N_s}\max \big(m^\alpha(u^i, y^k), v^k\big)$\;
        }
        $\mathcal{L}_p \gets \argmin_{\xi \in \mathcal{F}_p} \sum\limits_{i=1}^{|\grid|}l_p(p^i,\xi(u^i))$ \tcc*[r]{Regression}
        
\vspace{5pt}
\end{algorithm}

The `on-the-fly' method acquires evaluations $V_{N}(x)$ by performing nested Monte-Carlo integrations of the recurrence \eqref{eqn:iter_scheme}. with the initial condition $V_0(x) = \sup_{u \in \U} \varphi(u)$, where  Specifically, $\OTF (x, N, \varphi; \grid, N_s)$ approximates the q-value $\mathcal{V}_N (x, \cdot)$ defined by the recursion
\begin{equation}\label{eq:dp}
\begin{aligned}
\mathcal{V}_{1}(x,u) &= -\gamma\Upsilon(m^{\beta}(u, x), s^{\beta}(u, x)^2) + \ee \Big[\max \Big(m^\alpha(u, x_1), \varphi (x_1) \Big) \Big| x_0 = x \Big],\\
    \mathcal{V}_{n+1}(x,u) &= -\gamma\Upsilon(m^{\beta}(u, x), s^{\beta}(u, x)^2) + \ee \Big[\max \Big(m^\alpha(u, x_1), \sup_{u \in \U} \mathcal{V}_{n} (x_1, u) \Big) \Big| x_0 = x \Big],
\end{aligned}
\end{equation}
where $\varphi:\X \to \mathbb{R}$ is some given inital value as in Algorithm \ref{alg:vi}.
Setting $\varphi \equiv -\infty$, one can show by induction that $V_N(x) = \sup_{u \in \U} \mathcal{V}_N (x,u)$, c.f. \eqref{eqn:iter_scheme}. Notice also by comparing to Algorithm \ref{alg:lambda} that 
\[
\regress(x, \varphi; \grid, N_s) = \OTF (x, 1, \varphi; \grid, N_s).
\]
This motivates two modifications of Algorithm \ref{alg:core}. The first one, which we call `on-the-fly' in this paper, replaces calls to $\regress$ in lines 1 and 6 by calls to $\OTF (x, N, \varphi; \grid, N_s)$ with $\varphi \equiv -\infty$. We also increase the grid size $\grid$ and $N_s$ in the first execution to obtain a finer approximation of $\mathcal{V}_N$ than $(\mathcal{V}_n)_{n=1}^{N-1}$, in line with the discussion in Section \ref{sec:autoBML}.

The second modification of Algorithm \ref{alg:core} whose performance is not reported in this paper is motivated by the accuracy improvement when using $\regress(x, \wtl{V}; \grid', N_s')$ instead of $\wtl{V}$. In the same manner, calls to $\regress$ can be replaced by $\OTF (x, N, \wtl{V}; \grid', N_s')$ with $N \ge 2$ bringing further improvements but at a cost of significant increase in computational complexity.

The implementation of $\OTF$ is computationally intensive as it uses multiple nested Monte Carlo integrations and optimisations with the computational complexity exponential in $N$. Its advantage is that it does not require a good specification of the cloud $\cloud$ and precomputed value functions. 

\section{Validation}\label{sec:validation}

In this section we demonstrate applications of \algName to a synthetic example (to visualise its actions and performance) and a number of real datasets and models. Detailed information about settings for each example are collected in Appendix \ref{app:D}. All results were obtained using a desktop computer with Intel(R) Core(TM) i7-6700K CPU (4.4GHz), 32GB RAM and GeForce RTX 3090 GPU (24GB). In the synthetic example, Section \ref{subsec:syntetic}, we artificially slowed down computations in order to be able to observe differences in running times between hyperparameter choices.

For computation of 1D and 2D value function maps used in VI/Map algorithm, we employ a cloud $\cloud$ following general guidelines described in Subsection \ref{subsec:cloud}. However, in addition to those we enrich the cloud by first generating a small set of size $100$ of realistic unimodal shapes for the score and cost functions which we propagate via Kalman filter using a predefined control grid up to certain depth $D$. In 1D case we take depth $D=3$ and the cloud $\cloud$ contains $156,000$ states of which $1000$ are \textit{truths}. In 2D case, the value function is built using a cloud $\cloud$ extended by the above procedure with depth $D=10$; it contains $761,600$ states of which $1000$ are \textit{truths}. 

In VI/Map Algorithm \ref{alg:vi} the regression in line 6 is performed using Random Forest Regression (RFR) with $100$ trees. Functions $\sigma_H$ and $\sigma_T$ are taken as constant (independent of the hyperparameters $u$).

\subsection{A synthetic example}\label{subsec:syntetic}

\newcommand\ntrees{n_{\text{trees}}}

\begin{figure}[t]
    \centering
    \includegraphics[width=0.45\textwidth]{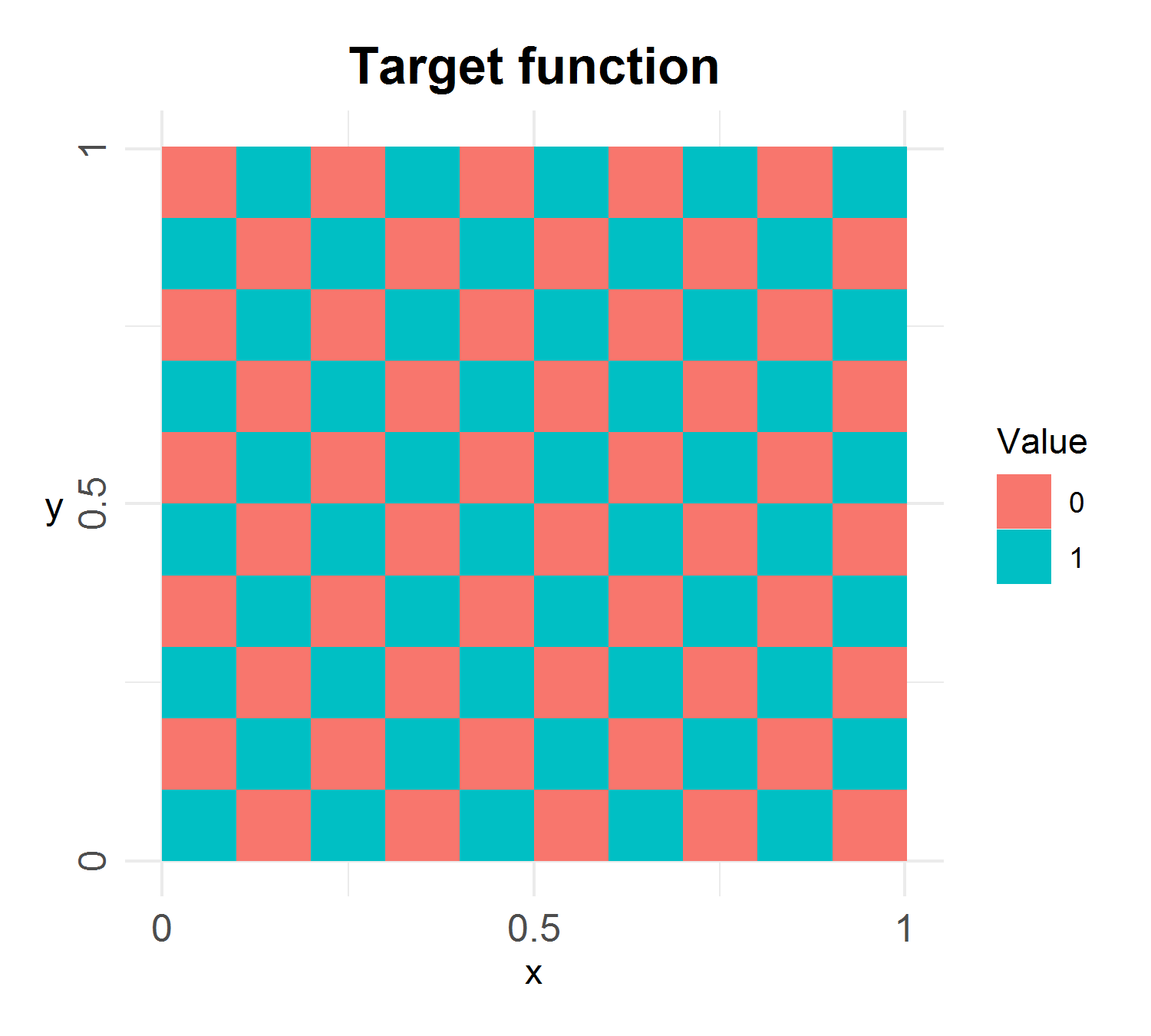}
    \caption{Function to be predicted in the synthetic example of Section \ref{subsec:syntetic}.}
    \label{fig:targetFunc}
    \centering
    \includegraphics[width=0.9\textwidth]{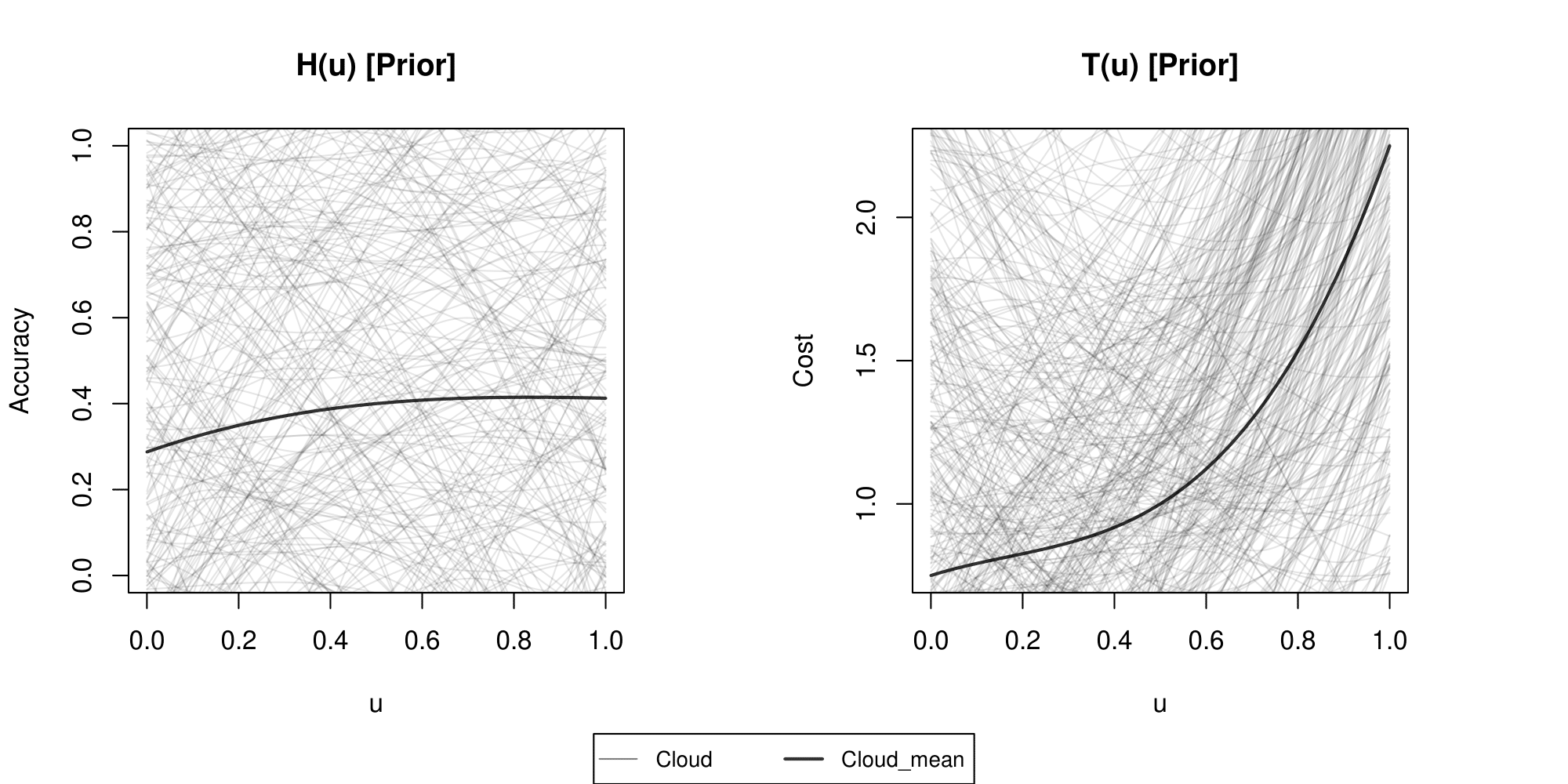}
    \caption{Score and cost priors for the synthetic example.}
    \label{fig:prior}
\end{figure}
\begin{figure}[p]
    \centering
    \begin{subfigure}[b]{1\textwidth}
        \includegraphics[width=1\linewidth]{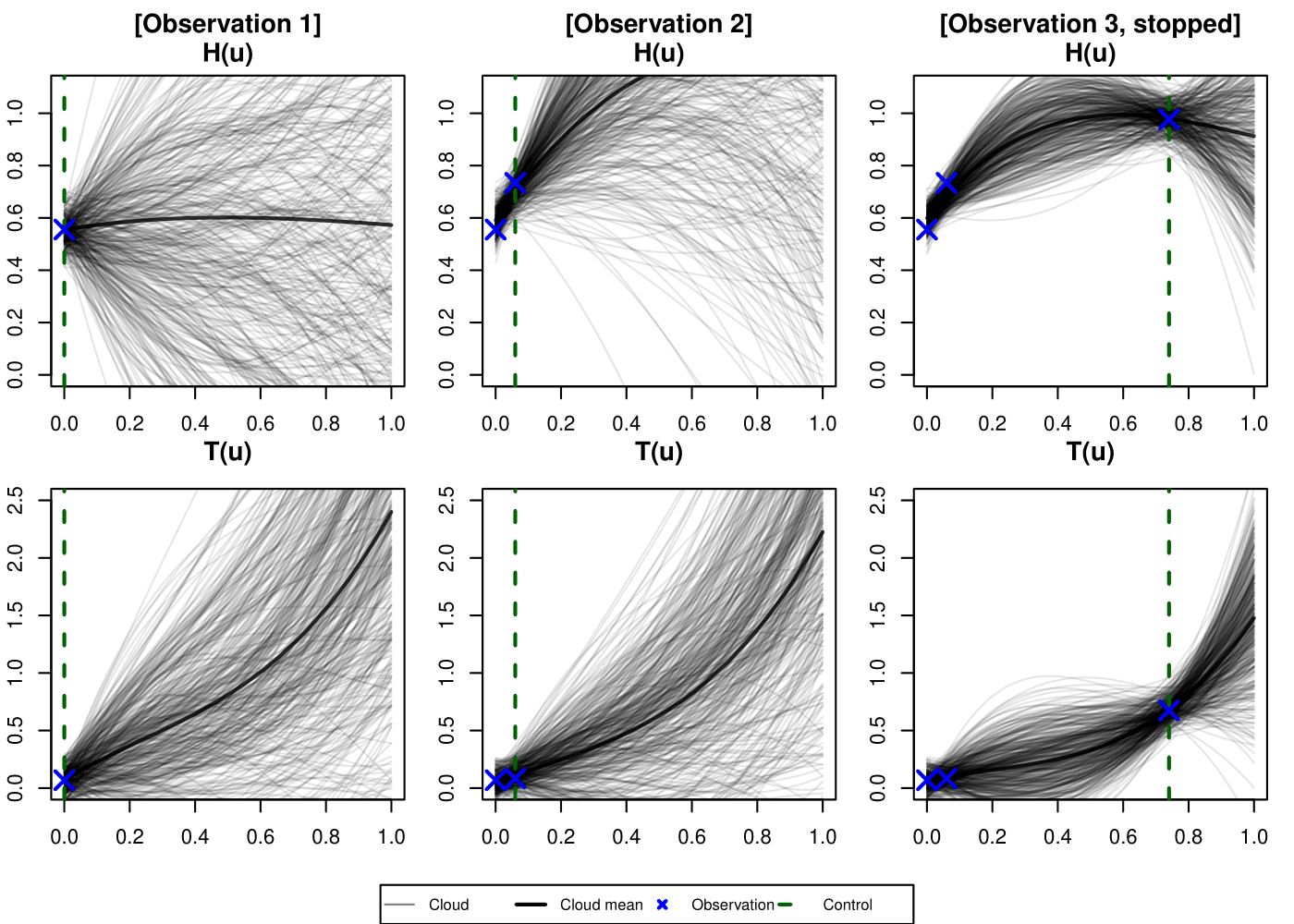}
        \caption{Evolution of posterior distributions of $H(u)$ and $T(u)$.}
    \end{subfigure}\\
    \vspace{10pt}
    \begin{subfigure}[b]{1\textwidth}
    \centering 
        \includegraphics[width=1\linewidth]{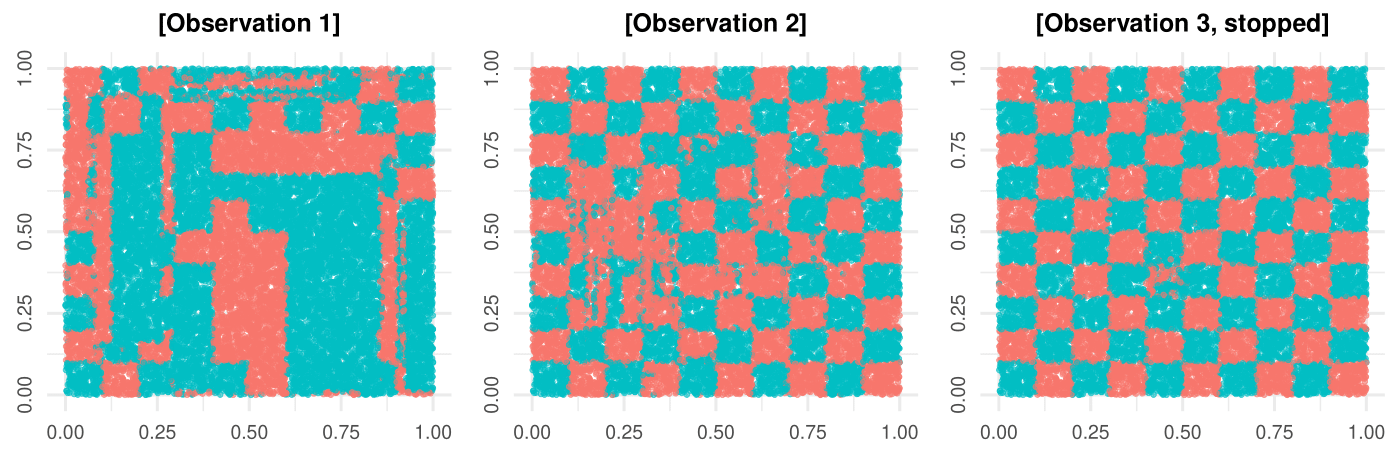}
        \caption{Validation performance on random points $(x,y) \sim \mathcal{U}(0,1)^2$.}
    \end{subfigure}
    \caption{One run for the synthetic classification example using \algName (VI/Map $N=2$, $\epsilon=0$).}
    \label{fig:synthresults}
\end{figure}

We illustrate steps of \algName algorithm on a synthetic prediction problem. To make results easier to interpret visually, we consider a one-dimensional hyperparameter selection problem where we control the number of trees in a random forest (RF) model and take scaled accuracy as the score. Observed accuracy $\hat h$ is mapped to $h$ by an affine mapping that maps $50\%$ model accuracy onto score $0$ and $100\%$ onto $1$, while the real running time $\hat t$ is mapped to $t$ by an affine mapping that maps $0$ seconds onto $0$ and $0.6$ seconds onto $1$:
\begin{equation*}
h= \frac{\hat{h} - 0.5}{0.5}, \qquad t= \frac{\hat{t}}{0.6}.
\end{equation*}
The task at hand is to predict outputs of a function 
\begin{equation*}
     f(x,y):= \big( \lfloor ax + 1\rfloor + \lfloor by + 1\rfloor \big) \mod 2, \quad (x,y) \in [0,1]\times[0,1],
\end{equation*}
with $a=b=10$ (Figure \ref{fig:targetFunc}), where $\lfloor \cdot \rfloor$ is the floor function: $\lfloor x \rfloor = \max\{m \in \mathbb{Z}\;|\; m \leq x\}$.

We initialise our synthetic experiment with a prior on $H(u)$ being relatively flat but with a large enough spread to cover well the range $(0,1)$ (see the left panel of Figure \ref{fig:prior}). For the cost function $T(u)$ we choose a pessimistic prior (see the right panel of Figure \ref{fig:prior}), which reflects our conservative belief that the cost is exponentially increasing as a function of the number of trees $n_{\max}$ (in practice, it is usually linear).
We fix $\gamma=0.16$ and $\epsilon = 0$ and set the observation errors $\sigma_H \equiv 0.05$ and $\sigma_T \equiv 0.1$ (standard deviation of the observation error for the score and the cost is assumed to be $5\%$ and $10\%$, respectively, independently of the choice of the hyperparameter). We relate the control $u\in[0,1]$ with the number of trees $\ntrees$ via the mapping:
\begin{equation}
    \ntrees(u) = \lfloor 1 + 99u \rfloor.
\end{equation}
Thus, the smallest considered random forest has $\ntrees(0)=1$ tree and the largest has $\ntrees(1)=100$ trees.

We present below performance of two algorithms: VI/Map, Algorithm \ref{alg:vi} ($N_s = 100$ with $N=2$ and $N=4$), and on-the-fly (OTF), Algorithm \ref{alg:OTF} ($N_s = 1000$ with $N=2$). We set $\grid = \{0, 0.01, 0.02, \ldots, 1\}$. Training set comprises $30,000$ points and the validation set contains $20,000$ points. In both sets, the points were randomly drawn from a uniform distribution on the square $[0,1]^2$.

\subsubsection{VI/Map results}

In this example, we use the 1D map described at the beginning of Section \ref{sec:validation}. 
We visualise results for the map produced with $N=2$ (recall that $N$ denotes the depth of recursion in the computation of the value function) and $\epsilon = 0$. We also include a summary of results for the $N=4$, $\epsilon = 0$ map (Table \ref{tab:synthresults_map_2}) and show that they are not significantly different from the $N=2$ case.

One run of the VI/Map ($N=2$, $\epsilon=0$) algorithm together with the evolution of posterior distributions is displayed in Figure \ref{fig:synthresults} with an accompanying summary Table \ref{tab:synthresults_map_1} which shows changes of important quantities across iterations. Notice that \algName stopped at iteration 3 with the control $u = 0.74$ (corresponding to $\ntrees = 74$) and the expected posterior score and the realised score of $\approx 0.98$ (accuracy $99\%$). Recall that the observed accuracy $h_n$ is burdened with a random error due to the randomness involved in training of the random forest and in validating the trained model. The closedness of the expected posterior score and the realised score at the time of stopping is not a rule and should not be expected in general (see examples further in this section).

We note that the VI/Map with $N=4$ produces a very similar run (Table \ref{tab:synthresults_map_2}), so additional expense at finding depth $N=4$ map is not necessary for this example. 

\begin{table}[tbp]
\begin{center}
\begin{tabular}{ |p{0.3cm}|p{0.9cm}|p{0.8cm}|p{1.3cm}|p{0.8cm}|p{2.3cm}|p{0.7cm}|}
\hline
\multicolumn{7}{|c|}{Synthetic example (VI/Map $N=2$, $\epsilon = 0$)} \\
\hline
$n$ & $h_n$ & $t_n$ & $\sum_{i=1}^n t_i$  & $u_{n-1}$ & $(\mu^{\alpha}_{n})^T\bphi(u_{n-1})$ & $V$ \\
\hline
1&0.556&0.068&0.068&0&0.555&0.75\\
2&0.734&0.083&0.151&0.06&0.697&1.147\\
3&0.977&0.675&0.826&0.74&0.98&0.974\\
\hline
\end{tabular}
\caption{\algName (VI/Map, $N=2$, $\epsilon = 0$) Results of one run for the synthetic example. The first column displays the iteration of the algorithm, the second columns the realised score $h_n$, the third the realised cost $t_n$, following by the total cost and the control $u_{n-1}$ applied in a given iteration and the expected posterior score $(\mu^{\alpha}_{n})^T\bphi(u_{n-1})$ corresponding to the control $u_{n-1}$ applied in iteration $n$. Finally $V$ shows the value function map evaluated at the posterior distribution after iteration $n$.}\label{tab:synthresults_map_1}
\end{center}
\end{table}

\begin{table}[tbp]
\begin{center}
\begin{tabular}{ |p{0.3cm}|p{0.9cm}|p{0.8cm}|p{1.3cm}|p{0.8cm}|p{2.3cm}|p{0.7cm}|}
\hline
\multicolumn{7}{|c|}{Synthetic example (VI/Map $N=4$, $\epsilon = 0$)} \\
\hline
$n$ & $h_n$ & $t_n$ & $\sum_{i=1}^n t_i$  & $u_{n-1}$ & $(\mu^{\alpha}_{n})^T\bphi(u_{n-1})$ & $V$ \\
\hline
1&0.56&0.074&0.074&0&0.559&0.803\\
2&0.738&0.093&0.167&0.07&0.707&1.129\\
3&0.981&0.653&0.82&0.79&0.984&0.974\\
\hline
\end{tabular}
\caption{\algName (VI/Map, $N=4$, $\epsilon = 0$) Results of one run for the synthetic example.}\label{tab:synthresults_map_2}
\end{center}
\end{table}

\paragraph{Non-zero $\epsilon$ value.}
We kept artificial damping parameter $\epsilon$, which accounts for errors in the value function map, equal to zero as the quality of the map was very good. It can also be observed in Table \ref{tab:synthresults_map_1} and \ref{tab:synthresults_map_2} that any reasonable value of $\epsilon$ would not change the course of training.

\subsubsection{On-the-fly results}

For the OTF ($N=2$) version of the algorithm we provide a summary of one run in Table \ref{tab:synthresults}. \algName stopped at iteration 3 with the control $u = 0.76$ (corresponding to $\ntrees = 76$) and the expected posterior score of $0.99$ (accuracy $99.5\%$) and realised score of $0.98$ (accuracy $99\%$). We observe that VI/Map results are very similar to the OTF results. In particular, both methods completed training in 3 iterations and the difference between final controls is very small.
\begin{table}[!ht]
\begin{center}
\begin{tabular}{ |p{0.3cm}|p{0.9cm}|p{0.8cm}|p{1.3cm}|p{0.8cm}|p{2.3cm}|p{0.7cm}|}
\hline
\multicolumn{7}{|c|}{Synthetic example (OTF, N=2)} \\
\hline
$n$ & $h_n$ & $t_n$ & $\sum_{i=1}^n t_i$  & $u_{n-1}$ & $(\mu^{\alpha}_{n})^T\bphi(u_{n-1})$ & $V$ \\
\hline
1&0.545&0.066&0.066&0&0.545&0.653\\
2&0.976&0.381&0.447&0.56&0.973&0.998\\
3&0.983&0.445&0.892&0.76&0.991&0.92\\
\hline
\end{tabular}
\caption{\algName (OTF, N=2) Results of one run for the synthetic example.}\label{tab:synthresults}
\end{center}
\end{table}

In the next section, we proceed to deploy \algName algorithm on real data, using convolutional neural networks.

\subsection{Convolutional neural network}\label{ssec:cnn}

In this section, we focus on a computationally intensive problem of classifying breast cancer specimen images for Invasive Ductal Carcinoma (IDC). IDC is the most common subtype of all breast cancers. Data \citep{janowczyk2016deep, cruz2014automatic} contains $277,524$ image patches of size $50\times50$, where $198,738$ are IDC negative and $78,786$ IDC positive.
    
We build a classification model using convolutional neural network (CNN). We optimise the following metaparameters: \textit{batch} (batch size) and $r$ (learning rate). CNN is run twice (2 \textit{epochs}) over a given set of training data and is evaluated on a holdout test data. Due to the fact that we run CNN for exactly 2 epochs (and not until a certain accuracy threshold is reached), the choice of the learning rate $r$ will have no effect on the running time. This is to demonstrate that \algName is also able to deal with cases where the cost does not depend on some of the optimised hyperparameters. Architecture of CNN is shown in Figure \ref{fig:cnnarchitecture} and described in Appendix \ref{app:breastcancerarch}.
\begin{figure}[tb]
    \centering
    \includegraphics[width=0.8\textwidth]{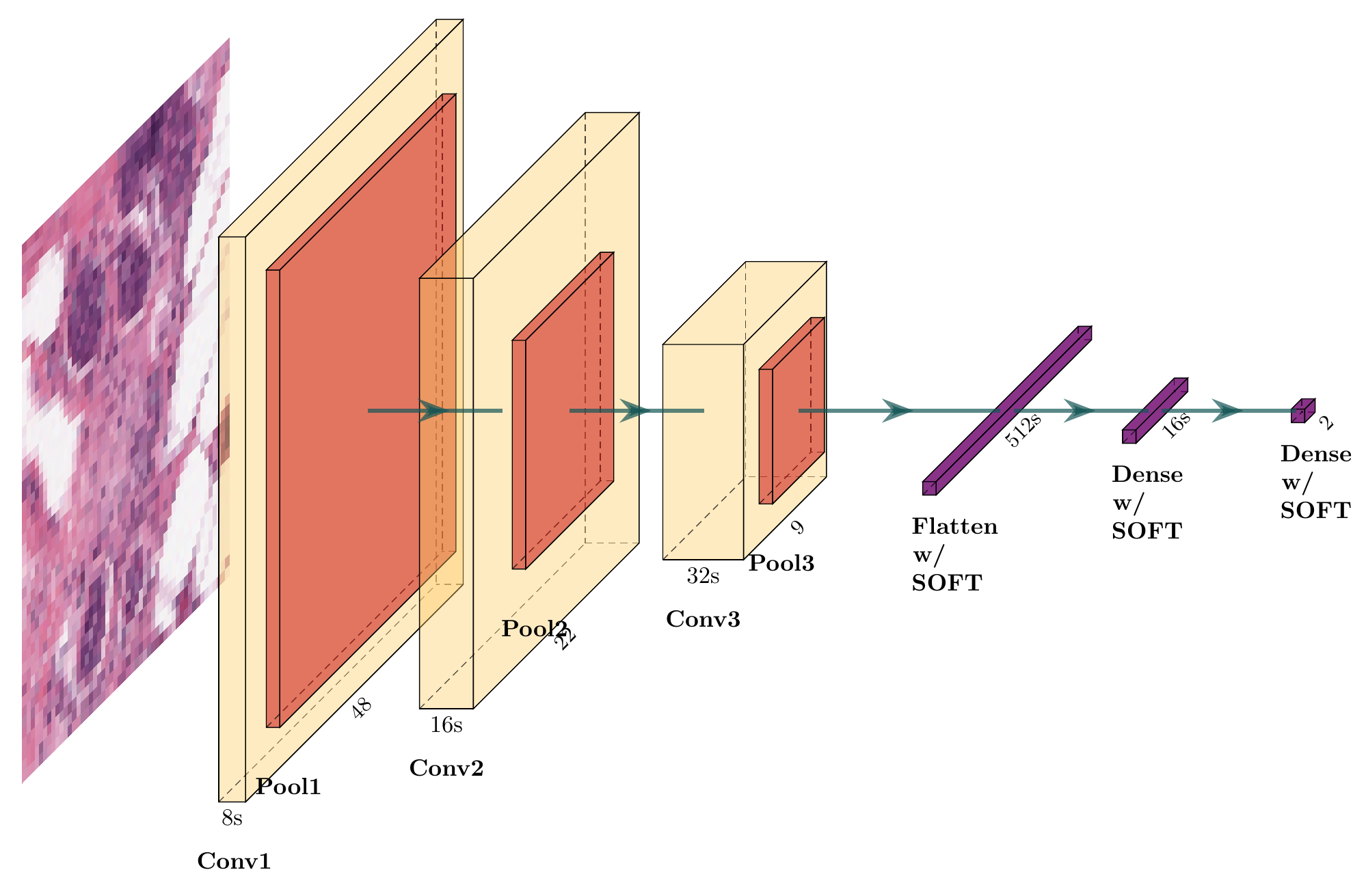}
    \caption{CNN architecture, where input is $50\times 50$ tissue image with $3$ colour channels and scaling parameter $s=1$ that affects the number of filters in convolution layers and sizes of flattened, densely connected arrays. A detailed information about the architecture of CNN can be found in Appendix \ref{app:breastcancerarch}.}
    \label{fig:cnnarchitecture}
\end{figure}

A naive tuning of CNN over a two-dimensional predefined grid of hyperparameters would be computationally expensive. In Sections \ref{sssec:otfcnnresults} and \ref{sssec:mapcnnresults} we present results of 2D \algName (VI/Map) and \algName (OTF), respectively, where nearly optimal controls are obtained in a relatively few iterations. However, we start by discussing 1D VI/map results on \textit{batch} and \textit{r} separately. Details about priors and parameter settings are collected in Appendix \ref{app:D}.

\subsubsection{VI/Map CNN \textit{batch} }\label{sssec:mapcnnbatch}

Due to limited computational power, we expedite CNN computational time by working on a subset of data. A training set consists of 5000 images with 50/50 split between the two categories, and a testing set is also balanced and contains 1000 images. We fix the learning rate $r=0.0007$.

We use the $N=2$ map with $\sigma_H = 0.15$, $\sigma_T = 0.1$ and $\gamma=0.16$. For this example as well as for others involving a neutral network, we choose relatively high $\sigma_H = 0.15$ to account for larger randomness in training for some choices of hyperparameters. For the sake of ilustration, we allow large batch sizes (relative to the size of the data set), which results in unstable model training. To mitigate the resulting significant variability in observed accuracy and enable interesting observations of the algorithm performance, for each $u$ the neural network is trained $5$ times and a median accuracy is outputted together with the total accumulated cost of training and testing $5$ neural networks. We map control $u \in [0,1]$ to batch size in the following way:
\begin{equation}\label{eqn:map_batch}
    \textit{batch}(u) = \lfloor \textit{batch}_{\min} + (\textit{batch}_{\max}-\textit{batch}_{\min})\cdot u \rfloor,
\end{equation}
which maps $0$ to $\textit{batch}_{\min}=10$ and $1$ to $\textit{batch}_{\max}=200$ linearly. The accuracy and the running time (in minutes) are mapped into score and cost as follows
\begin{equation}\label{eqn:CNN_rescale}
h:= \frac{\hat{h} - 0.45}{0.35}, \qquad t:= \frac{\hat{t}}{7.5}.
\end{equation}

Results for this problem, dubbed 1D CNN \textit{batch}, can be inspected in Table \ref{tab:cnnbatchresults}. We also visualise final posterior distributions of $H(\cdot)$ and $T(\cdot)$ in Figure \ref{fig:batch_final}. The algorithm returned the control corresponding to batch size of $89$ with realised posterior score of $0.81$ (corresponding to the accuracy $73\%$) and expected posterior score of $0.83$ (accuracy $74\%$). 

\begin{figure}[t]
    \centering
    \includegraphics[width=0.9\textwidth]{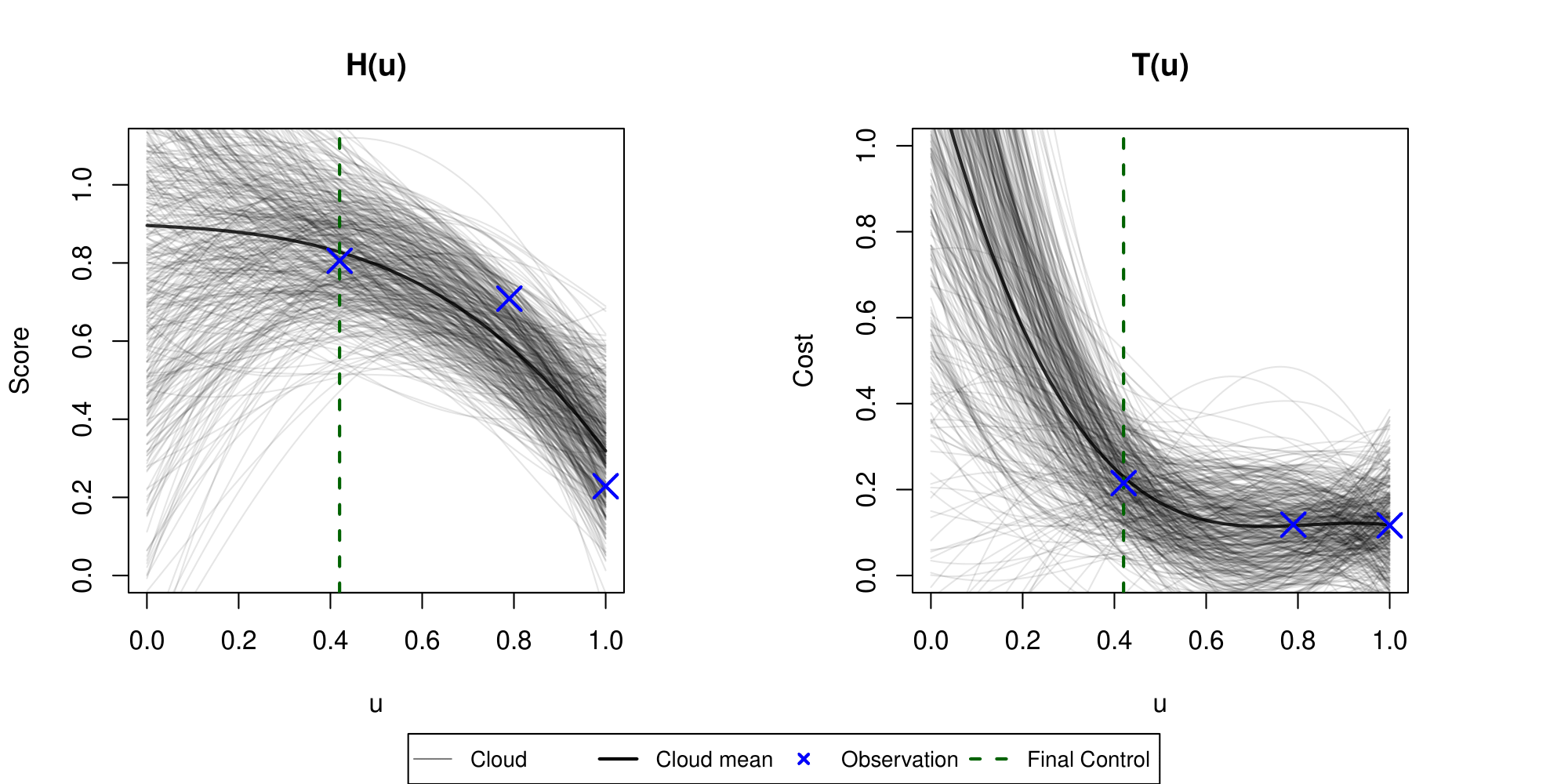}
    \caption{Visualisation of a final posterior distribution of a 1D CNN \textit{batch} example.}
    \label{fig:batch_final}
\end{figure}

\begin{table}[t]
\begin{center}
\begin{tabular}{ |p{0.3cm}|p{0.9cm}|p{0.8cm}|p{1.3cm}|p{0.8cm}|p{2.3cm}|p{0.7cm}|}
\hline
\multicolumn{7}{|c|}{CNN example; \textit{batch} (VI/map, N=2, $\epsilon = 0$)} \\
\hline
$n$ & $h_n$ & $t_n$ & $\sum_{i=1}^n t_i$  & $u_{n-1}$ & $(\mu^{\alpha}_{n})^T\bphi(u_{n-1})$ & $V$ \\
\hline
1&0.229&0.117&0.117&1&0.23&0.526\\
2&0.709&0.118&0.235&0.79&0.619&0.978\\
3&0.806&0.215&0.45&0.42&0.828&0.8\\
\hline
\end{tabular}
\caption{\algName (VI/Map, N=2, $\epsilon=0$) results on 1D CNN example with \textit{batch} as control.}\label{tab:cnnbatchresults}
\end{center}
\end{table}

\subsubsection{VI/Map CNN \textit{r} }\label{sssec:mapcnnlr}

Here we explore another 1D case using learning rate \textit{r} as the control. We fix the batch size $66$. It is known that the learning rate does not influence the computational cost of the neural net trained over fixed number of epochs. However, by having the batch size $66$ we make each run of the neural net relatively expensive in the context of these examples ($\approx$ $2.7$ minutes). Therefore, our goal is still to obtain an optimal control in as few iterations as possible, to save computational resources.
\begin{figure}[t]
    \centering
    \begin{subfigure}[b]{0.9\textwidth}
      \centering
      \includegraphics[width=0.9\textwidth]{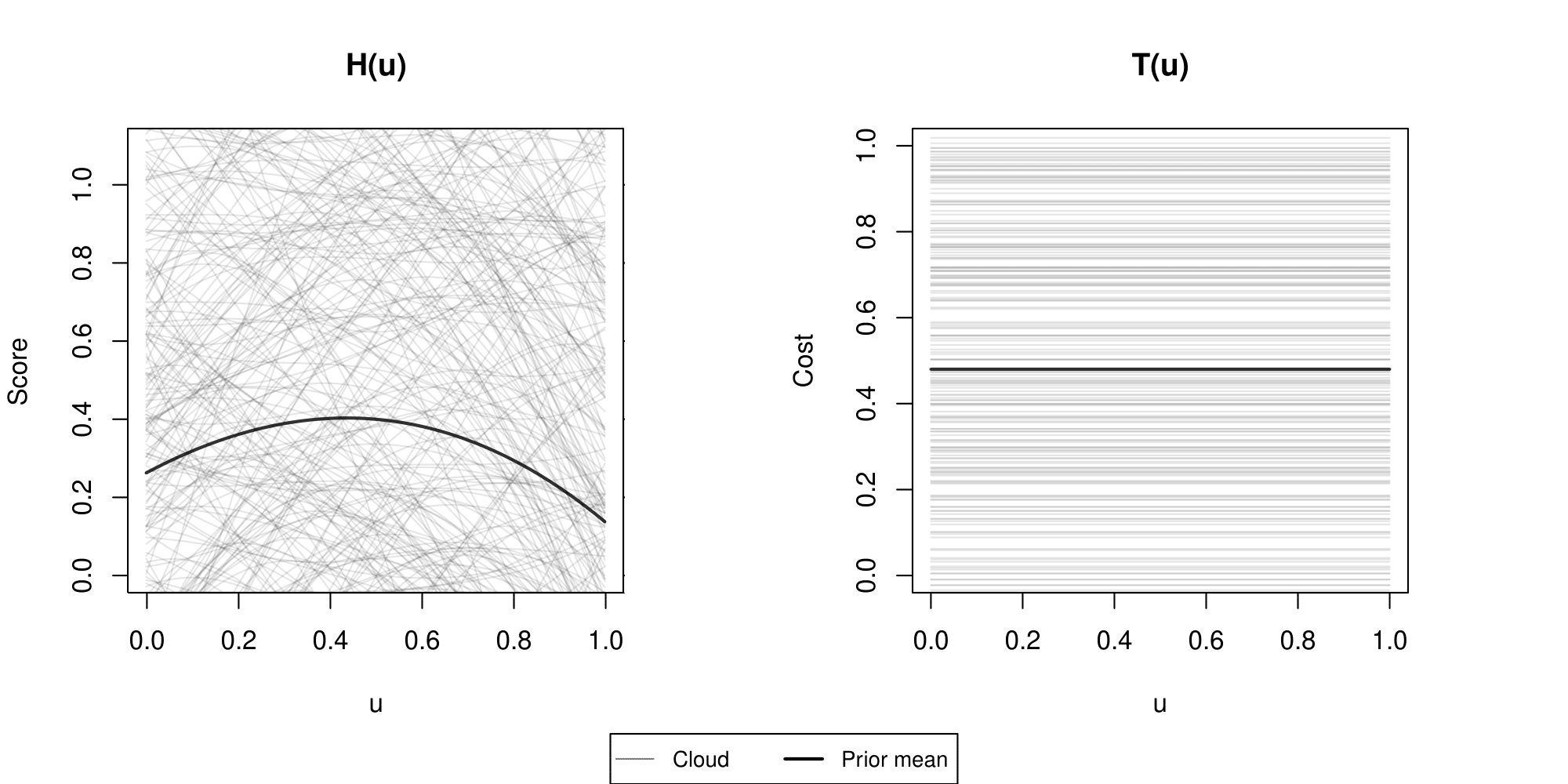}
      \caption{Prior distributions for $H$ and $T$.}
      \label{fig:priors_r}
    \end{subfigure}
    \begin{subfigure}[b]{0.9\textwidth}    
     \centering
     \includegraphics[width=0.9\textwidth]{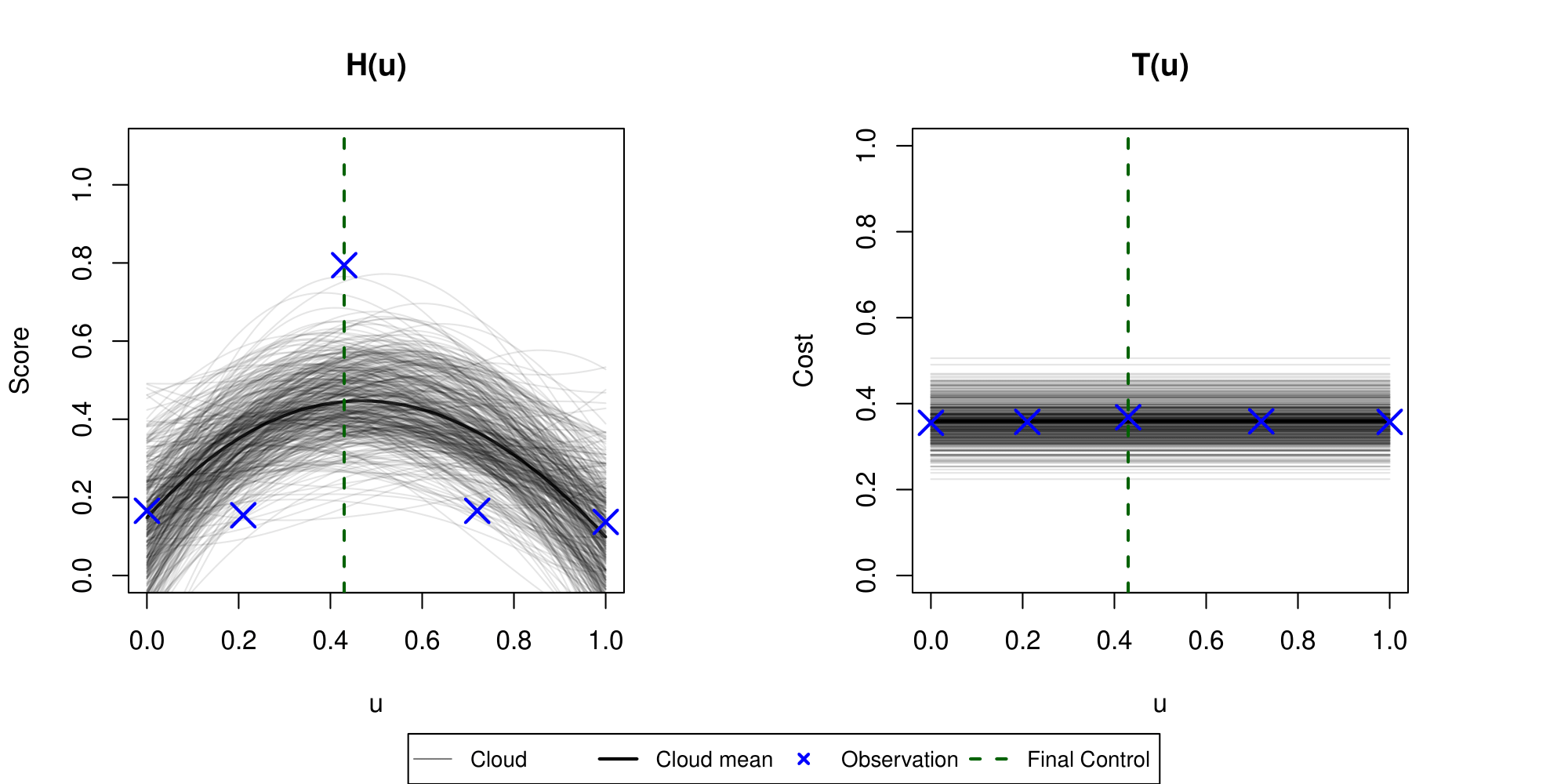}
     \caption{Final posterior distribution of $H$ and $T$.}
     \label{fig:posterior_r}
    \end{subfigure}
    \caption{Prior and posterior distributions for 1D CNN \textit{r} example.}
\end{figure}
We keep the same mappings for $h$ and $t$ as above, and the same $\sigma_H = 0.15$, $\sigma_T = 0.1$, $\gamma=0.16$. Using common knowledge about the effect of the learning rate on accuracy and running time, we modify the prior for $H(\cdot)$ so that the mean is unimodal and has a maximum in the interior of $(0,1)$ and the prior for $T(\cdot)$ is flat (Figure \ref{fig:priors_r}).
The mapping of control into the learning rate is linear on the $\log(r)$ scale, i.e., we set
\begin{equation}\label{eqn:r_map}
r(u) = \exp\Big\{\log(r_{\min}) + \big(\log(r_{\max}) - \log(r_{\min})\big) \cdot u)\Big\},
\end{equation}
where $r_{\min}=10^{-5}$ and $r_{\max}=0.1$.

In the run displayed in Table \ref{tab:cnnrresults}, the final control corresponds to $r = 0.0005$ with the observed score of $h_5 = 0.794$ (corresponding to the accuracy $72\%$) and the expected posterior score of $(\mu^{\alpha}_{3})^T\bphi(0.43) = 0.445$ (accuracy $61\%$). The final posterior distributions with observed scores and costs indicated by crosses are shown in Figure \ref{fig:posterior_r}. The final output of the algorithm is good although the posterior distribution on the left panel poorly fits the observed data. The shape of the score curve is so peaked that it cannot be represented accurately with a degree three polynomial (our choice of basis functions). However, \algName stops when the posterior mean score (not the observed score) exceeds the value function, so it is sufficient that the posterior distribution indicates the positioning of hyperparameters for which the true maximum score is attained.

\begin{table}[bt]
\begin{center}
\begin{tabular}{ |p{0.3cm}|p{0.9cm}|p{0.8cm}|p{1.3cm}|p{0.8cm}|p{2.3cm}|p{0.7cm}|}
\hline
\multicolumn{7}{|c|}{CNN example; \textit{r} (VI/map, N=2, $\epsilon = 0$)} \\
\hline
$n$ & $h_n$ & $t_n$ & $\sum_{i=1}^n t_i$  & $u_{n-1}$ & $(\mu^{\alpha}_{n})^T\bphi(u_{n-1})$ & $V$ \\
\hline
1&0.137&0.357&0.357&1&0.137&0.647\\
2&0.154&0.357&0.715&0.21&0.163&0.31\\
3&0.166&0.359&1.074&0.72&0.202&0.269\\
4&0.166&0.355&1.429&0&0.137&0.188\\
5&0.794&0.368&1.797&0.43&0.445&0.389\\
\hline
\end{tabular}
\caption{\algName (VI/Map, N=2, $\epsilon=0$) results on 1D CNN example with $r$ as control.}\label{tab:cnnrresults}
\end{center}
\end{table}

\subsubsection{2D \algName (VI/Map) with CNN }\label{sssec:otfcnnresults}

\begin{figure}[t]
\begin{subfigure}[b]{\textwidth}
    \centering
    \includegraphics[width=0.8\textwidth]{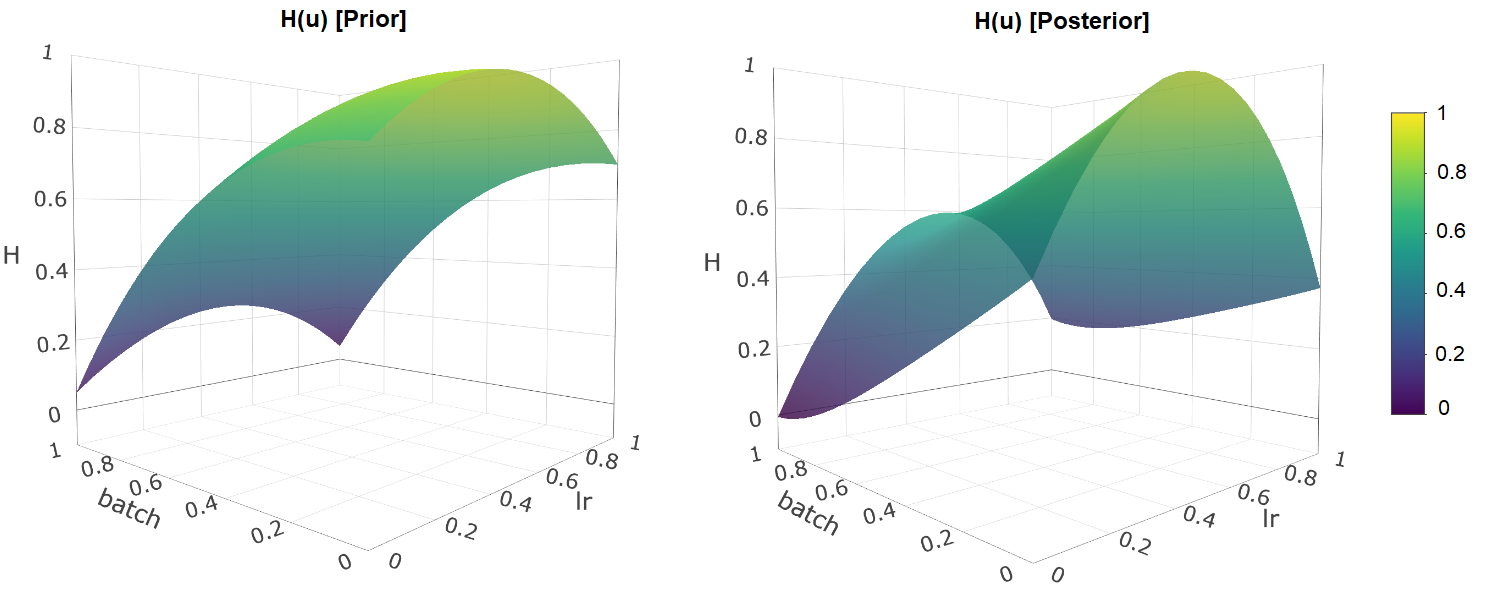}
    \caption{Surfaces corresponding to prior (left) and posterior (right) means of the score.}
    \label{fig:2dprior}
\end{subfigure}
\begin{subfigure}[b]{\textwidth}
    \centering
    \includegraphics[width=0.8\textwidth]{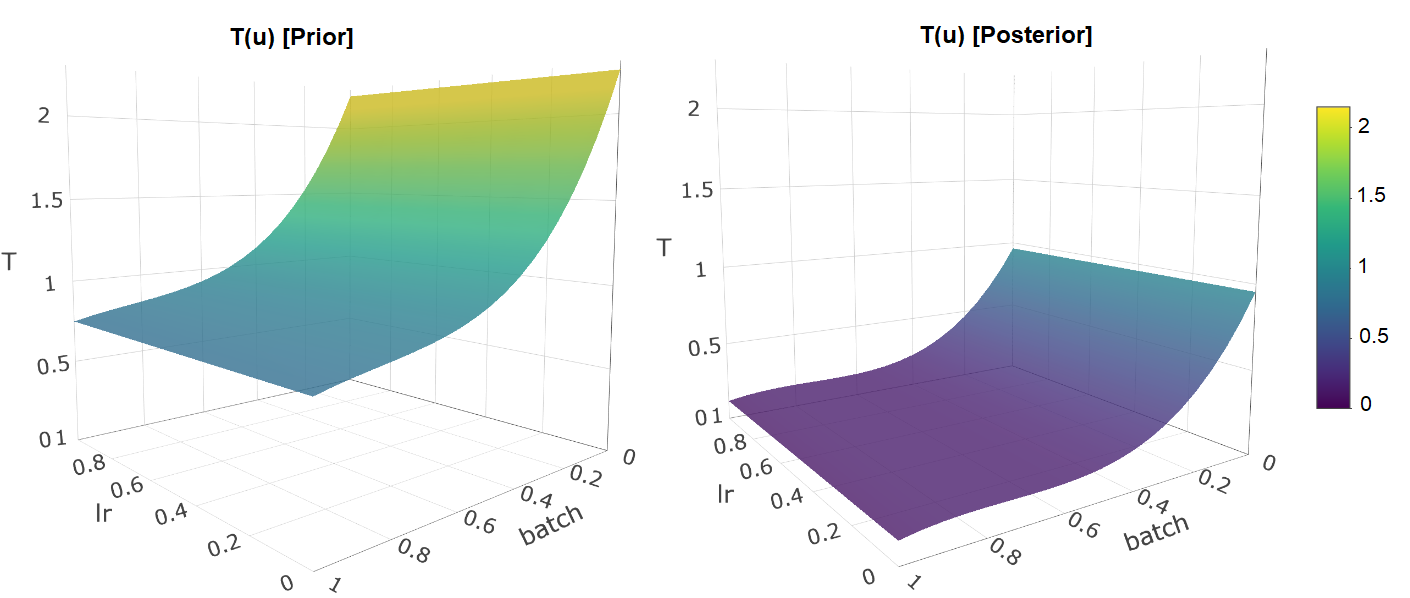}
    \caption{Surfaces corresponding to prior (left) and posterior (right) means of the cost.}
    \label{fig:2dposterior}
\end{subfigure}
\caption{Prior and posterior distributions for a 2D CNN run from Table \ref{tab:2dcnnmap}}\label{fig:2d}
\end{figure}

We initialise our 2D CNN experiment with a prior on score $H(u_1,u_2)$ being unimodal (left panel of Figure \ref{fig:2dprior}). For the cost function $T(u_1,u_2)$ we choose a pessimistic prior (left panel of Figure \ref{fig:2dposterior}), which reflects a conservative belief that the cost is exponentially decreasing in the batch size and indifferent to the choice of the learning rate.
As before, for each pair $(u_1,u_2)$ the neural network is trained $5$ times to mitigate instabilities with large \textit{batch}. We map the median raw accuracy $\hat{h}$ and the raw time $\hat{t}$ (in minutes) via \eqref{eqn:CNN_rescale}.
We map the control $u_1$ into the learning rate via \eqref{eqn:r_map} with $r_{\min}=10^{-5}$ and $r_{\max}=0.1$. The second control, $u_2$, is mapped into the batch size via \eqref{eqn:map_batch} with $\textit{batch}_{\min} = 10$ and $\textit{batch}_{\max} = 200$. Those mappings are identical as in the 1D cases studied above.
\begin{table}[tb]
\begin{center}
\begin{tabular}{ |p{0.2cm}|p{0.9cm}|p{0.8cm}|p{1.1cm}|p{0.8cm}|p{0.8cm}|p{3.5cm}|p{0.8cm}| }
\hline
\multicolumn{8}{|c|}{CNN example; \textit{r} and \textit{batch} (VI/Map, N=3, $\epsilon=0.02$)} \\
\hline
$n$ & $h_n$ & $t_n$ & $\sum_{i=1}^n t_i$  & $u_{1,n-1}$ & $u_{2,n-1}$ & $(\mu^{\alpha}_{n})^T\bphi(u_{1,n-1},u_{2,n-1})$ & $V$ \\
\hline
1&0.639&0.166&0.166&0.47&0.36&0.64&0.936\\
2&0.16&0.18&0.346&1&0.44&0.181&0.9\\
3&0.148&0.186&0.531&0&0.44&0.186&0.807\\
4&0.849&0.198&0.729&0.46&0.51&0.666&0.86\\
5&0.577&0.109&0.838&0.5&0.89&0.559&0.847\\
6&0.373&0.154&0.992&0.24&0.71&0.464&0.831\\
7&0.502&0.155&1.147&0.71&0.71&0.533&0.765\\
8&0.711&0.135&1.282&0.63&1&0.592&0.679\\
9&0.953&0.439&1.721&0.5&0.15&0.837&0.776\\
\hline
\end{tabular}
\caption{An example run of \algName on CNN example with the 2D control space. $u_{1,n-1}$ corresponds to the learning rate $r$ and $u_{2, n-1}$ corresponds to the batch size \textit{batch}.}\label{tab:2dcnnmap}
\end{center}
\end{table}

Results for an example run of 2D \algName (VI/map) can be inspected in Table \ref{tab:2dcnnmap}. The final control of $(u_1,u_2)=(0.5,0.15)$ corresponds to $r=0.001$ and the batch size of $38$, with the expected posterior accuracy of $74\%$ and the realised accuracy of $78\%$ (compared to the accuracy $84\%$ obtained in \cite{cruz2014automatic} which used full dataset in computations). The surfaces showing the mean of the posterior distributions for $H$ and $T$ are shown on right panels of Figure \ref{fig:2d}. 

Table \ref{tab:mapsummary} presents a summary of average performance of 40 runs on this 2D neural network problem. During 40 VI/Map runs, we recorded a small standard deviation $0.06$ for the final control $u_{1,n}$ (\textit{r}) suggesting that the quality of the model is sensitive to the choice of the learning rate. On the other hand, the standard deviation for the final control $u_{2,n}$ (\textit{batch}) is much bigger suggesting weaker sensitivity to this parameter. It should, however, be remarked that the training of Neural Networks tends to be less stable when batch size is large relative to the total sample size, so good models could have been obtained `by chance'.

\begin{table}[tb]
\begin{center}
\begin{tabular}{ |p{0.8cm}|p{1.45cm}|p{1.2cm}|p{0.8cm}|p{0.8cm}|p{1.6cm}|p{3.55cm}|p{0.8cm}|}
\hline
\multicolumn{8}{|c|}{Summary of (VI/Map, N=3, $\epsilon=0.02$) \algName runs for the 2D CNN problem} \\
\hline
$Run$ & $Final$ $h_n$ & $\sum_{i=1}^n t_i$  & $u_{1,n-1}$ & $u_{2,n-1}$ & $Tot.$ $Iter.$& $(\mu^{\alpha}_{n})^T\bphi(u_{1,n-1},u_{2,n-1})$ & $V$ \\
\hline
A-40 & 0.782  &2.021         & 0.46 & 0.295 & 8.45& 0.784    &0.748\\
     & (0.15) & (0.55) & (0.06)   & (0.23)  & (1.8)&     \\
\hline
\end{tabular}
\caption{Summary results of an avarage (A-40) of $40$ runs (standard deviation is given in brackets) for the 2D CNN problem. Average controls $(u_{1,n-1}, u_{2, n-1}) = (0.46, 0.295)$ correspond to the learning rate $0.0007$ and the batch size $66$.}\label{tab:mapsummary}
\end{center}
\end{table}

\subsubsection{2D \algName (OTF) with CNN }\label{sssec:mapcnnresults}

The OTF version of the \algName with a 2D CNN problem is significantly affected by the relatively high noise present in observations of the score and we observed its tendency to stop more unpredictably. We present summaries of three OTF runs (Table \ref{tab:otfsummary}). Care should be taken when comparing these results to the ones obtained by \algName(VI/Map) because, due to computational demands, only $N=2$ version of OTF is ran, while for the VI/Map version we use $N=3$ map with $1$ OTF step at the end (as always for VI/Map runs, see Algorithm \ref{alg:lambda}), essentially making it $N=4$ for the comparison with OTF. It can, therefore, be concluded that depth $N=2$ is insufficient to obtain reliable results in 2D case and one needs to use depth $N=4$. Recall that $N=2$ version of OTF performed well in the synthetic 1D example. 
\begin{table}[tb]
\begin{center}
\begin{tabular}{ |p{0.8cm}|p{1.5cm}|p{1cm}|p{0.8cm}|p{0.8cm}|p{1.8cm}|p{3.55cm}|p{0.8cm}|}
\hline
\multicolumn{8}{|c|}{Summary of OTF N=2 \algName runs for the 2D CNN problem} \\
\hline
$Run$ & $Final$ $h_n$ & $\sum_{i=1}^n t_i$  & $u_{1,n-1}$ & $u_{2,n-1}$ & $Tot.$ $Iter.$& $(\mu^{\alpha}_{n})^T\bphi(u_{1,n-1},u_{2,n-1})$ & $V$ \\
\hline
1 & 0.814&3.729&0.38&0&10&0.879&0.718\\
2 & 0.771&0.996&0.48&0.14&4&0.693&0.692\\
3 & 0.592&2.844&0.38&0.28&8&0.579&0.546\\
\hline
\end{tabular}
\caption{Final results of 3 OTF N=2 runs with 2D CNN problem.}\label{tab:otfsummary}
\end{center}
\end{table}

\subsection{\algName validation on popular datasets}\label{sssec:validationdatasets}

We provide summary statistics for the evaluation of \algName's performance on 3 popular datasets employing 3 different machine learning models. In tables, we provide in brackets: the accuracy for column \textit{Final $h_n$}, the total running time for column $\sum_{i=1}^n t_i$, and the corresponding value of the hyperparameter for column $u$ ($u_1$, $u_2$). Recall that details of parameter mappings and implementation are collected in Appendix \ref{app:D}.
\bigskip

\noindent
\textbf{HIGGS Data Set} \cite{baldi2014searching}. This is a classification problem to distinguish between a signal process which produces Higgs bosons and a background process which does not. The data set contains 11 million rows. Training over the whole dataset is time consuming, so we aim to find an optimal \textit{sample size} to be used for training and testing with Random Forest (\verb|ranger|) given the trade-off between the score and the computational cost. Results are collected in Table \ref{tab:validation_higgs}. The final accuracy of $72\%$ puts the model within the TOP 10 performing models according to \textit{openml.org/t/146606} statistics for Higgs data set.
\begin{table}[!h]
\centering
\begin{tabular}{ |p{0.9cm}|p{1.6cm}|p{2.3cm}|p{3.7cm}|p{1.35cm}|p{2cm}|p{0.6cm}|}
\hline
\multicolumn{7}{|c|}{1D (VI/Map, N=3, $\epsilon=0.02$) \algName with Random forest on Higgs data} \\
\hline
\textit{Data} & \textit{Final} $h_n$ & $\sum_{i=1}^n t_i$  & $u$ & \textit{Tot. Iter.} & $(\mu^{\alpha}_{n+1})^T\bphi(u)$ & $V$ \\
\hline
Higgs & 0.6 (72\%) & 0.96 (4.8 \text{min}) & 0.13 (\textit{samp size} = 2.4\%) & 7 & 0.6 & 0.58\\
\hline
\end{tabular}
\caption{Random forest on `Higgs' data set using \textit{sample size} as a control.}\label{tab:validation_higgs}
\end{table}
\bigskip

\noindent
\textbf{Intel image data}. This is image data of Natural Scenes around the world obtained from Kaggle (\url{https://www.kaggle.com/puneet6060/intel-image-classification}). This data was part of the challenge created by (\url{https://datahack.analyticsvidhya.com}) and was sourced by Intel company. The data contains around 25K images of size $150\times150$ distributed under 6 categories. The task is to create a classification model. For this example we use categorical cross entropy loss \citep{cox1958} ($\mathcal{X}$ - population, $x$ - sample, $\mathcal{C}$ - set of all classes, $c$ - a class)
\begin{equation}
-\frac{1}{N} \sum_{x \in \mathcal{X}} \sum_{c \in \mathcal{C}} \ind{x \in c} \log {\mathbb{P}(x \in c)}
\end{equation} 
to construct a score (instead of the accuracy). The above entropy loss allows us to monitor overfitting of the CNN while accuracy itself does not provide us with this information. In this example we choose the number of \textit{epochs} in CNN as the control. We obtain results for two architectures. In one architecture (Resnet50) we use weights of the widely popular resnet50 model, while in other one (Plain) we use architecture that does not rely on pretrained models. Details of architectures of Plain and Resnet50 models are provided in Appendix \ref{app:intelarch}. Results are displayed in Table \ref{tab:validation_intel_noresnet}. We note that our Resnet50 model performance produces a result that is $3\%$ short of the challenge winner, who obtained a model with $96\%$ accuracy via a different architecture and ensamble averaging.

\begin{table}[tb]
\begin{center}
\begin{tabular}{ |p{1.4cm}|p{1.8cm}|p{2.3cm}|p{2.8cm}|p{1.35cm}|p{2cm}|p{0.6cm}|}
\hline
\multicolumn{7}{|c|}{1D (VI/Map, N=3, $\epsilon=0.02$) \algName with CNN on Intel image data} \\
\hline
\textit{Init.} & \textit{Final} $h_n$ & $\sum_{i=1}^n t_i$  & $u$ & \textit{Tot. Iter.} & $(\mu^{\alpha}_{n+1})^T\bphi(u)$ & $V$ \\\hline
Plain    & 0.97 (86\%) & 0.75 (37 \text{min}) & 0.54 (\textit{epochs} = 19) & 9 & 0.99 & 0.96\\
Resnet50 & 0.93 (93\%) & 0.75 (39 \text{min}) & 0.48 (\textit{epochs} = 17) & 9 & 0.96 & 0.93\\
\hline
\end{tabular}
\caption{CNN on Intel data set using \textit{epochs} as a control using two architectures. In brackets in the first column we show the test set accuracy while the score corresponds to the cross-entropy loss.}\label{tab:validation_intel_noresnet}
\end{center}
\end{table}

We note that a more efficient optimisation of a parameter such as \textit{epochs} is available. Having a fixed holdout (test) dataset we are able to evaluate the performance of the model on a test set after each training epoch, without introducing any data leakages to the model. Therefore, if a cost of the model evaluation on the holdout set is relatively small, compared to the overall training cost per epoch, observations of the score ($h$) and the cost ($t$) are available at each epoch. Optimiser's task would therefore simplify to optimally stopping the training when the optimum trade-off between the score and the cost is achieved. However, the flow of information would be different than in our model and a new modelling approach would be needed. Furthermore, in a Bayesian setting like ours, it is required that the error of each observation of the score and the cost is independent of the previous ones. This is clearly violated in an optimisation task with incremental improvements of the model. Nevertheless, this is an interesting direction of research which we retain for the future. Note also similarities to the learning curve methods used by, e.g., \cite{Chandra2017}.
\bigskip

\noindent
\textbf{Credit card fraud data} (\url{https://www.kaggle.com/mlg-ulb/creditcardfraud}). This dataset contains credit card transactions from September 2013. It is highly unbalanced with only $492$ frauds out of $284,807$ transactions. We build an autoencoder (AE) \citep{kramer1991nonlinear} which is trained on non-fraudulent cases. We attempt to capture fraudulent cases by observing the reconstruction error. 

We parametrise the network with two parameters: \textit{code} which specifies the size of the smallest layer and \textit{scale} is the number of layers in the encoder (the decoder is symmetric). The largest layer of the encoder has the size $\lceil 20 \times \text{scale} \rceil$ and there are $\lceil \textit{scale} \rceil$ layers in the encoder spaced equally between the largest layer and the coding layer ($\lceil x \rceil$ denotes the ceiling function, i.e., the smallest integer dominating $x$). So an AE described by code equal to $1$ and scale equal to $1.9$ has hidden layers of the following sizes: $38, 19, 1, 19, 38$. We consider \textit{code} and \textit{scale} between $1$ and $10$.

For 1D example, the optimisation parameter is the \textit{scale} with $\textit{code}=2$. For 2D example, we optimise \textit{scale} and \textit{code}.

Given that fraudulent transactions constitute only $0.172\%$ of all transactions, we do not focus on a pure accuracy of the model, but on a good balance between \emph{precision} and \emph{recall} which is weighted together by the widely used generalized F-score \citep{dice1945measures}:
\begin{equation}
F_{\beta} = (1+\beta^2) \frac{\text{\textit{precision}} \times \text{\textit{recall}}}{(\beta^2 \times \text{\textit{precision}}) + \text{\textit{recall}}}.
\end{equation}
We pick $\beta=6$ as we are interested in capturing most of the fraudulent transactions (we observe that the \emph{recall} is at least $80\%$). In practice $\beta$ is chosen somewhat arbitrarily; however it also can be used as an optimisation parameter \citep{klinger2009user}.

Results for 1D example are displayed in Table \ref{tab:validation_ae}.
\begin{table}[tb]
\begin{center}
\begin{tabular}{ |p{1cm}|p{2.6cm}|p{2.5cm}|p{2.7cm}|p{1.6cm}|p{2cm}|p{0.7cm}|}
\hline
\multicolumn{7}{|c|}{1D (VI/Map, N=3, $\epsilon=0.02$) \algName with AE on Credit Card Fraud data} \\
\hline
$Data$ & $Final$ $h_n$ & $\sum_{i=1}^n t_i$  & $\hat{u}$ & $Tot.$ $Iter.$& $(\mu^{\alpha}_{n+1})^T\bphi(u)$ & $V$ \\
\hline
Fraud & 0.99 ($F_{\beta}=0.74$) & 1.89 (42.9 \text{min}) & 0.01 (\textit{scale}=1.09) & 4 & 0.93 & 0.87\\
\hline
\end{tabular}
\caption{AE on `Credit Card Fraud' dataset using \textit{scale} as a control.}\label{tab:validation_ae}
\end{center}
\end{table}
We observe that in this particular 1D example the optimal control of $0.01$ corresponds to a shallow AE architecture, suggesting that good performance can be achieved with a small in size AE architecture.

The results for the 2D version of this example are shown in Table \ref{tab:validation_ae2d}. \algName decides to stop at the most shallow architecture as well, surprisingly revealing that even choosing the size of the \textit{code} equal to $1$ leads to similarly good results. 

\begin{table}[tb]
\begin{center}
\begin{tabular}{ |p{1cm}|p{1.9cm}|p{1.4cm}|p{1.5cm}|p{1.5cm}|p{1.6cm}|p{2cm}|p{0.8cm}|}
\hline
\multicolumn{8}{|c|}{2D (VI/Map, N=3, $\epsilon=0.02$) \algName with AE on Credit Card Fraud data} \\
\hline
$Data$ & $Final$ $h_n$ & $\sum_{i=1}^n t_i$  & $\hat{u_1}$ & $\hat{u_2}$ & $Tot.$ $Iter.$& $(\mu^{\alpha}_{n+1})^T\bphi(u)$ & $V$ \\
\hline
Fraud & 0.97 & 2.868 & 0  & 0 & 8 & 0.922& 0.897\\
&($F_{\beta}=0.73$) & (62 \text{min}) & (\textit{scale}=1) & (\textit{code}=1) & & &\\
\hline
\end{tabular}
\caption{AE on `Credit Card Fraud' dataset using \textit{scale} and \textit{code} as a controls.}\label{tab:validation_ae2d}
\end{center}
\end{table}

\noindent
\textbf{Conclusions.} Our validation demonstrated that \algName performs well in hyperparameter selection problems for diverse datasets, machine learning techniques and objectives. VI/Map version of the algorithm is preferred to OTF. The value function map can be precomputed in advance and used for a range of problems.

\appendix
\appendixpage
\setcounter{equation}{0}

\section{Dynamics of the observable Markov process} \label{app:dynamics}
We derive joint dynamics of the filters $(\cA_n)$ and $(\cB_n)$, and observable processes $(h_n)$ and $(t_n)$. Random variables $\balpha$ and $\bbeta$ are independent and observations are independent given the control $(u_n)$. Therefore, the dynamics of their filters $(\cA_n)$ and $(\cB_n)$ are independent given the control $(u_n)$. 

Since the process $(\cA_n)$ takes values in the space of Gaussian distributions, its dynamics can be fully described by the dynamics of the corresponding stochastic process representing means $(\mu^\alpha_n)$ and covariances $\Sigma^\alpha_n$. The dynamics of $(\mu^\alpha_n, \Sigma^\alpha_n)$ is Markovian, time-homogeneous and depends on the control $(u_n)$: for any Borel measurable sets $D_1 \subset \er^J$ and $D_2 \subset \er^{J \times J}$ (recall that $J$ is the number of basis functions)
\begin{equation}\label{eqn:6}
\begin{aligned}
&\prob\big( \mu^\alpha_{n+1} \in D_1, \Sigma^\alpha_{n+1} \in D_2 \big| \mu^\alpha_n = \mu, \Sigma^\alpha_n = \Sigma, u_n = u)\\[5pt]
&=
\int_{\er} \ind{\mu^\alpha_1(\hat h,u) \in D_1} \ind{\Sigma^\alpha_1(u) \in D_2} \exp\bigg(-\frac12 \frac{\big(\hat h - \mu^T \bphi(u) \big)^2}{\bphi(u)^T \Sigma \bphi(u) + \sigma^2_H(u)}\bigg) d\hat h,
\end{aligned}
\end{equation}
where $\mu^\alpha_1(\hat h, u)$ and $\Sigma^\alpha_1 (u) $ are obtained from formulas \eqref{eqn:filter_a} with $u_0=u$, $h_1 = \hat h$ and $\Sigma^\alpha_0=\Sigma$, $\mu^\alpha_0 = \mu$:
\begin{equation*}
 \begin{aligned}
(\Sigma^\alpha_{1}(u))^{-1} &= (\Sigma)^{-1} + \frac{1}{\sigma_H(u)^2} \bphi(u) \bphi(u)^T,\\
\mu^\alpha_{1}(\hat h,u) &= \mu + \frac{1}{\sigma_H(u)^2} \Sigma^\alpha_{1} \bphi(u) \big( \hat h - \bphi(u)^T \mu^\alpha_n \big).
 \end{aligned}
\end{equation*}
The above formula for $\Sigma^\alpha_{1}(u)$ is convenient for proofs. However, the inversion of matrices is numerically undesirable, so in our implementation we use an equivalent formula \citep[Eq. (4.7.5)]{Bensoussan2018}
\[
\Sigma^\alpha_{1}(u) = \Sigma - \frac{\Sigma \bphi(u) \bphi(u)^T \Sigma}{\bphi(u)^T \Sigma \bphi(u) + \sigma^2_H(u)}.
\]

It follows from \eqref{eqn:6} that the simulation of the Markov process $(\mu^\alpha_n, \Sigma^\alpha_n)$ may be performed with an intermediate step of generating $h_{n+1}$. Given $(\mu^\alpha_n, \Sigma^\alpha_n)$ and the control $u_n$, we first draw $h_{n+1}$ from the normal distribution with the mean $(\mu^\alpha_n)^T \bphi(u_n)$ and the variance $\bphi(u_n)^T \Sigma^\alpha_n \bphi(u_n) + \sigma^2_H(u_n)$. Then we compute $\mu^\alpha_{n+1} = \mu^\alpha_1(h_{n+1}, u_n)$ and $\Sigma^\alpha_{n+1} = \Sigma^\alpha_1(u_n)$ inserting $\mu = \mu^\alpha_n$ and $\Sigma = \Sigma^\alpha_n$.

The process $(\cB_n)$ is described as above by the mean $(\mu^\beta_n)$ and the covariance matrix $(\Sigma^\beta_n)$. We consider the dynamics of $(\mu^\beta_n, \Sigma^\beta_n, t_n)$ as $t_n$ shows up explicitly in the objective function \eqref{eqn:4}. For Borel measurable sets $D_1 \subset \er^K$,  $D_2 \subset \er^{K \times K}$ and $D_3 \in \er$
\begin{equation}\label{eqn:6a}
\begin{aligned}
&\prob\big( \mu^\beta_{n+1} \in D_1, \Sigma^\beta_{n+1} \in D_2,  t_{n+1} \in D_3 \big|  \mu^\beta_n = \mu, \Sigma^\beta_n = \Sigma, u_n = u, t_n = t)\\[5pt]
&=
\int_{\er} \ind{\mu^\beta_1(\hat t, u) \in D_1} \ind{\Sigma^\beta_1(u) \in D_2} \ind{\hat t \in D_3^{\phantom{\beta}}} \exp\bigg(-\frac12 \frac{\big(\hat t - \mu^T \bpsi(u) \big)^2}{\bpsi(u)^T \Sigma \bpsi(u) + \sigma^2_T(u)}\bigg) d\hat t,
\end{aligned}
\end{equation}
where $\mu^\beta_1(\hat t, u)$ and $\Sigma^\beta_1(u)$ is given by \eqref{eqn:filter_b} with $u_0=u$, $t_{1}=\hat t$ and $\Sigma^\beta_0=\Sigma$, $\mu^\beta_0 = \mu$. Notice that the transition function above does not depend on the value $t_n=t$; it is included for notational coherence. Observe also that the simulation of the process $(\mu^\beta_n, \Sigma^\beta_n, t_n)$ can be performed in an analogous way as for $(\mu_n^\alpha, \Sigma_n^\alpha)$.

\section{Auxiliary estimates}\label{app:aux}

Without further mention, we assume that all matrices are symmetric and positive definite. We write $\lambda_{\max} (\Sigma)$ for the largest eigenvalue of the matrix $\Sigma$ (or equivalently, the operator norm of this martix). Analogously, $\lambda_{\min}(\Sigma)$ denotes the smallest eigenvalue. Recall also the notation for the Kalman filter \eqref{eqn:filter_a} and \eqref{eqn:filter_b}. 

\begin{lemma}\label{lem:est1}
For any $u_0 \in \U$, $\lambda_{\max} (\Sigma^\alpha_1) \le \lambda_{\max} (\Sigma^\alpha_0)$ and $\lambda_{\max} (\Sigma^\beta_1) \le \lambda_{\max} (\Sigma^\beta_0)$, $\prob$-a.s.
\end{lemma}
\begin{proof}
Follows directly from \eqref{eqn:filter_a} and observation that $\lambda_{\min} \big((\Sigma^\alpha_1)^{-1}\big) \ge \lambda_{\min} \big((\Sigma^\alpha_0)^{-1}\big)$.
\end{proof}

\begin{lemma} \label{lem:A1}
Let a function $\varphi:\mathbb{R}^J \times \mathbb{R}^{J\times J} \mapsto \mathbb{R}$ satisfy
\[
|\varphi(\mu, \Sigma)| \le a + b \|\mu\| + c \lambda_{\max}^p (\Sigma)
\]
for some $a, b, c, p > 0$. Under Assumptions \ref{ass:bounded} and \ref{ass:sigma_u}, we have
\[
\ee \Big[ \big| \varphi(\mu^\alpha_1, \Sigma^\alpha_1) \big| \Big| \mu^\alpha_0 = \mu, \Sigma^\alpha_0=\Sigma, u_0 = u \Big] \le \tilde a + \tilde b \|\mu\| + \tilde c \lambda_{\max}^q (\Sigma)
\]
for some $\tilde a, \tilde b, \tilde c > 0$ and $q = \max(3/2, p)$.
\end{lemma}
\begin{proof}
Let $A = \sup_{u \in \U} \max_j \phi_j(u)$. For any $u \in \U$, skipping in the notation the conditioning on $\mu^\alpha_0 = \mu, \Sigma^\alpha_0=\Sigma, u_0 = u$, we have
\begin{equation}\label{eqn:A1}
\ee [ \varphi(\mu^\alpha_1, \Sigma^\alpha_1) ] 
\le a + b \ee [ \| \mu^\alpha_1 \|] + c \ee [ \lambda^p_{\max} (\Sigma^\alpha_1)]
\le a + b \ee [ \| \mu^\alpha_1 \|] + c \lambda^p_{\max} (\Sigma),
\end{equation}
where the second inequality follows from Lemma \ref{lem:est1}. It remains to estimate the norm of $\mu^\alpha_1$. From \eqref{eqn:6} and the triangle inequality for the Euclidean norm, we have
\begin{align*}
\ee [ \| \mu^\alpha_1 \|] 
&\le 
\| \mu \| + \frac{1}{\sigma_H(u)} \ee_{h \sim N\big(\mu^T \bphi(u),\, \bphi(u)\Sigma\bphi(u) + \sigma^2_H(u)\big)} \Big[ \big\| \Sigma^\alpha_1 \bphi(u) (h - \mu^T \bphi(u)) \big\| \Big],
\end{align*}
where $\Sigma^\alpha_1$ depends on $h_1=h$ and $u_0=u$ via formula \eqref{eqn:filter_a}. Applying $\big\| \Sigma^\alpha_1 \bphi(u) (h - \mu^T \bphi(u)) \big\| \le \lambda_{\max}(\Sigma) A |h - \mu^T \bphi(u)|$, we obtain
\begin{align*}
\ee [ \| \mu^\alpha_1 \|] 
&\le 
\| \mu \| + \frac{\lambda_{\max}(\Sigma) A}{\sigma_H(u)} \ee_{h \sim N\big(\mu^T \bphi(u),\, \bphi(u)\Sigma\bphi(u) + \sigma^2_H(u)\big)} [ |h - \mu^T \bphi(u)| ]\\
&=
\| \mu \| + \frac{\lambda_{\max}(\Sigma) A}{\sigma_H(u)} \ee_{h \sim N\big(0,\, \bphi(u)\Sigma\bphi(u) + \sigma^2_H(u)\big)} [ |h| ]\\
&\le
\| \mu \| + \frac{\lambda_{\max}(\Sigma) A}{\sigma_H(u)} \sqrt{\ee_{h \sim N\big(0,\, \bphi(u)\Sigma\bphi(u) + \sigma^2_H(u)\big)} [ h^2 ]}\\
&=
\| \mu \| + \frac{\lambda_{\max}(\Sigma) A}{\sigma_H(u)} \sqrt{\bphi(u)\Sigma\bphi(u) + \sigma^2_H(u)}\\
&\le
\| \mu \| + \frac{\lambda_{\max}(\Sigma) A}{\sigma_H(u)} \sqrt{A^2 \lambda_{\max}(\Sigma) + \sigma_H^2(u)}\\
&\le
\| \mu \| + \frac{\lambda_{\max}(\Sigma) A}{\sigma_H(u)} \big(A \lambda_{\max}^{1/2} (\Sigma) + \sigma_H(u)\big)\\
&=
\| \mu \| + \lambda^{3/2}_{\max}(\Sigma)  \frac{A^2}{\sigma_H(u)} + \lambda_{\max}(\Sigma) A\\
&=
\| \mu \| + b_1 \lambda^{3/2}_{\max}(\Sigma) + a_1,
\end{align*}
for $a_1 = \lambda_{\max}(\Sigma) A$ and $b_1 = A^2 / \sigma_H(u)$. Above the second inequality is by Jensen's inequality, in the third inequality we use again Assumption \ref{ass:sigma_u}, and the fourth inequality is due to $\sqrt{x^2+y^2} \le x + y$ for $x,y \ge 0$. Inserting the above estimate into \eqref{eqn:A1} completes the proof.
\end{proof}

\begin{lemma} \label{lem:A2}
Let a function $\varphi:\mathbb{R}^J \times \mathbb{R}^{J\times J} \times \mathbb{R}^K \times \mathbb{R}^{K\times K}\mapsto \mathbb{R}$ satisfy
\[
|\varphi(\mu^\alpha, \Sigma^\alpha, \mu^\beta, \Sigma^\beta)| \le a + b (\|\mu^\alpha\| + \|\mu^\beta\|) + c \big( \lambda_{\max}^p (\Sigma^\alpha) + \lambda_{\max}^p (\Sigma^\beta) \big)
\]
for some $a, b, c, p > 0$. Under Assumptions \ref{ass:bounded} and \ref{ass:sigma_u}, we have 
\[
|\cT \varphi(\mu^\alpha, \Sigma^\alpha, \mu^\beta, \Sigma^\beta) | \le
\tilde a + \tilde b (\|\mu^\alpha\| + \|\mu^\beta\|) + \tilde c \big( \lambda_{\max}^q (\Sigma^\alpha) + \lambda_{\max}^q (\Sigma^\beta) \big),
\]
where $\tilde a, \tilde b, \tilde c > 0$ and $q = \max (p, 3/2)$.
\end{lemma}
\begin{proof}
We use the representation \eqref{eqn:reform} of the functional $\cT$. For any $u \in \U$ with $B = \sup_{u \in \U} \max_l \psi_l(u)$ and $C = \sup_{u \in \U} \sigma^2_T(u)$, direct calculations yield
\begin{align*}
\Tcost \big(\bpsi(u)^T \mu^\beta, \bpsi(u)^T \Sigma^\beta \bpsi(u) + \sigma^2_T(u)\big) 
&\le
\frac{1}{2\pi} \sqrt{\bpsi(u)^T \Sigma^\beta \bpsi(u) + \sigma^2_T(u)} + \bpsi(u)^T \mu^\beta \\
&\le
\frac{1}{2\pi} \sqrt{\bpsi(u)^T \Sigma^\beta \bpsi(u)} + \frac{1}{2\pi} \sigma_T(u) + \| \bpsi(u)\| \|\mu^\beta\|\\
&\le
\frac{1}{2\pi} \| \bpsi(u) \| \lambda_{\max}^{1/2}(\Sigma^\beta)+  \frac{1}{2\pi} C + \| \bpsi(u)\| \|\mu^\beta\|\\
&\le
\frac{1}{2\pi} B \lambda_{\max}^{1/2}(\Sigma^\beta)+  \frac{1}{2\pi} C + B \|\mu^\beta\|.
\end{align*}

Since
\[
\Big| \max \big( (\mu^\alpha_1)^T \bphi(u), \varphi(\mu^\alpha_1, \Sigma^\alpha_1, \mu^\beta_1, \Sigma^\beta_1) \big) \Big|
\le
\big| (\mu^\alpha_1)^T \bphi(u)\big| + \big|\varphi(\mu^\alpha_1, \Sigma^\alpha_1, \mu^\beta_1, \Sigma^\beta_1) \big|,
\]
we have, supressing in the notation the conditioning on $\mu^\alpha_0 = \mu^\alpha$, $\Sigma^\alpha_0=\Sigma^\alpha$, $\mu^\beta_0 = \mu^\beta$, $\Sigma^\beta_0=\Sigma^\beta$, $u_0 = u$,
\begin{align*}
&E \big[ \max \big( (\mu^\alpha_1)^T \bphi(u), \varphi(\mu^\alpha_1, \Sigma^\alpha_1, \mu^\beta_1, \Sigma^\beta_1) \big) \big]\\
&\le
E \big[ \big| (\mu^\alpha_1)^T \bphi(u)\big| \big] + E \big[ \big|\varphi(\mu^\alpha_1, \Sigma^\alpha_1, \mu^\beta_1, \Sigma^\beta_1) \big| \big]\\
&\le
\bigg[A \big( \| \mu^\alpha \| + b_1 \lambda^{3/2}_{\max}(\Sigma^\alpha) + a_1 \big) \bigg]+ 
\bigg[a_2 + b_2 \big(\|\mu^\alpha\| + \|\mu^\beta\| \big) + c_1 \big(\lambda_{\max}^q (\Sigma^\alpha) + \lambda_{\max}^q (\Sigma^\beta)\big)\bigg],
\end{align*}
where the estimate of the first term was derived in the proof of Lemma \ref{lem:A1} with $A = \sup_{u \in \U} \max_j \phi_j(u)$, and the estimate of the second terms follows analogously as in Lemma \ref{lem:A1}. 
Inserting the above bound and the bound for $\Tcost$ into \eqref{eqn:reform} completes the proof.

\end{proof}

\section{Proofs}\label{app:proofs}

\begin{proof}[Proof of Theorem \ref{thm:separation}]
We have
\[
\ee[H(u_{\tau-1})] = \sum_{n=1}^\infty \ee[ \ind{\tau=n} H(u_{n-1}) ].
\]
Conditioning $n$-th term of the above sum on $\ef_n$ yields
\begin{align*}
\ee \Big[ \ind{\tau = n} H(u_{n-1}) \Big| \ef_n \Big] 
&= 
\ee \Big[ \ind{\tau = n} \balpha^T \bphi(u_{n-1}) \Big| \ef_n \Big]
=
\ind{\tau = n} \ee \Big[\balpha^T \bphi(u_{n-1}) \Big| \ef_n \Big]\\
&=
\ind{\tau = n} \ee \Big[\balpha^T \Big| \ef_n \Big] \bphi(u_{n-1})
=
\ind{\tau = n} (\mu^\alpha_n)^T \bphi(u_{n-1}),
\end{align*}
where we used that $\{\tau = n\} \in \ef_n$ and $u_{n-1}$ is $\ef_{n-1}$-measurable. Hence, using the tower property of conditional expectation the objective function \eqref{eqn:3} can be rewritten as
\begin{align*}
\sum_{n=1}^\infty \ee[ \ind{\tau=n} H(u_{n-1}) ] - \ee \Big[ \sum_{n=1}^\tau (t_n)^+ \Big]
&=
\sum_{n=1}^\infty \ee\big[ \ee[\ind{\tau=n} H(u_{n-1})|\ef_n]\big] - \ee \Big[ \sum_{n=1}^\tau (t_n)^+ \Big]\\
&=
\sum_{n=1}^\infty \ee\big[ \ind{\tau = n} (\mu^\alpha_n)^T \bphi(u_{n-1})\big] - \ee \Big[ \sum_{n=1}^\tau (t_n)^+ \Big]\\
&=
\ee \Big[ (\mu^\alpha_\tau)^T \bphi(u_{\tau-1}) - \sum_{n=1}^\tau (t_n)^+ \Big],
\end{align*}
which completes the proof.
\end{proof}

\begin{lemma}\label{lem:upper_bound}
The value function defined in \eqref{eqn:7} is bounded from above by
\[
v^* (\mu^\alpha, \Sigma^\alpha) := \ee[\|\balpha\|] \sup_{u \in \U} \|\bphi(u)\|, \qquad \text{where $\balpha \sim N(\mu^\alpha, \Sigma^\alpha)$},
\]
and bounded from below by
\[
v_*(\mu^\alpha, \Sigma^\alpha, \mu^\beta, \Sigma^\beta) := - v^*(\mu^\alpha, \Sigma^\alpha) - \gamma \ee[\|\bbeta\|] \sup_{u \in \U} \|\bpsi(u)\|, \qquad \text{where $\bbeta \sim N(\mu^\beta, \Sigma^\beta)$},
\]
and $\|\cdot\|$ denotes the Euclidean norm. Furthermore, functions $v^*, v_*$ are finite-valued.
\end{lemma}
\begin{proof}
Recall from the proof of Theorem \ref{thm:separation} that the first term of \eqref{eqn:7} can be equivalently written as 
\[
\ee[(\mu^\alpha_\tau)^T \bphi(u_{\tau-1})] = \ee[\balpha^T \bphi(u_{\tau-1})],
\]
where $\balpha \sim N(\mu^\alpha, \Sigma^\alpha)$. By assumption \ref{ass:bounded}, the vector $\bphi(u_{\tau-1})$ has uniformly bounded entries, so $C := \sup_{u \in \U} \|\bphi(u)\| < \infty$. By Cauchy-Schwarz inequality, 
\[
|\balpha^T \bphi(u_{\tau-1})| \le \|\bphi(u_{\tau-1})\| \|\balpha\| \le C \|\balpha\|, 
\]
so
\[
\ee[\balpha^T \bphi(u_{\tau-1})] \le \ee[|\balpha^T \bphi(u_{\tau-1})|] \le C \ee[\|\balpha\|] = v^*(\mu^\alpha, \Sigma^\alpha).
\]
Recalling that $\balpha$ is multivariate Gaussian, we have $\ee[\|\balpha\|] < \infty$ and $v^*(\mu^\alpha, \Sigma^\alpha) < \infty$. Since the second term in \eqref{eqn:7} is non-positive, $V$ is bounded from above by $v^*(\mu^\alpha, \Sigma^\alpha) < \infty$. 

Taking $\tau = 1$ and any $u_0 \in \U$ gives a finite lower bound on $V$:
\begin{align*}
V(\mu^\alpha, \Sigma^\alpha, \mu^\beta, \Sigma^\beta) 
&\ge 
-\ee[|\balpha^T \bphi(u_{0})|] - \gamma \ee[(\bbeta^T \bpsi(u_0))^+] \\
&\ge
-v^*(\mu^\alpha, \Sigma^\alpha) - \gamma \ee[|\bbeta^T \bpsi(u)|]\\
&\ge
-v^*(\mu^\alpha, \Sigma^\alpha) - \gamma \ee[\|\bbeta\|]\, \sup_{u \in \U} |\bpsi(u)|
= v_*.
\end{align*}

\end{proof}
\medskip

\begin{proof}[Derivation of formula \eqref{eqn:reform}]
Recall that the distribution of $t_1$ is normal with the mean $(\mu^\beta)^T \bpsi(u)$ and the variance $\bpsi(u)^T \Sigma^\beta \bpsi(u) + \sigma^2_T(u)$. The result follows from the following direct calculation: for $Y \sim N(m, s^2)$
\[
\ee[Y^+] = \ee [Y \ind{Y \ge 0}] = \frac{s}{\sqrt{2\pi}} e^{-\frac{m^2}{2s^2}} + m \Phi \Big(\frac{m}{s}\Big) = \Tcost (m, s^2).
\]
\end{proof}

\begin{proof}[Proof of Theorem \ref{thm:2}]
In the proof we use the representation \eqref{eqn:reform} of operator $\cT$. From the proof of Lemma \ref{lem:A2}, it follows that
\begin{equation}\label{eqn:9}
\big |V_1( \mu^\alpha,  \Sigma^\alpha,  \mu^\beta,  \Sigma^\beta) \big|
\le a + b (\| \mu^\alpha\| + \| \mu^\beta\|) + c \big(\lambda^{3/2}_{\max} ( \Sigma^\alpha) + \lambda^{3/2}_{\max} ( \Sigma^\beta) \big)
\end{equation}
for some $a, b, c \ge 0$. Hence $\cT V_1$ is well defined. By induction and Lemma \ref{lem:A2}, an estimate of the form \eqref{eqn:9} holds for $V_N$, so $\cT V_N$ is well defined. The bound also ensures that infinite values are not admitted, which completes the proof of Assertion 1. Assertion 2 follows from a standard theory of stochastic control \citep{Bertsekas1978}.

The fact that $V_N$ are non-decreasing in $N$ is a direct consequence of the set of controls in \eqref{eqn:7} growing with $N$. Analogously, $V_N \le V$. This implies that the limit $V_\infty = \lim_{N \to \infty} V_N$ exists and $V_\infty \le V$. By Assumption \ref{ass:bounded}, the random variable $H^* = \sup_{u \in \U} |H(u)|$ is integrable ($\ee[H^*]<\infty$), so by the dominated convergence theorem we have
\[
\lim_{N \to \infty} \ee \big[H(u_{(\tau \wedge N)-1})\big] = \ee \big[H(u_{\tau-1})\big] < \infty
\]
for any admissible strategy $(\tau, (u_n))$. By the monotone convergence theorem
\[
\lim_{N \to \infty} \ee \Big[\sum_{n=1}^{\tau \wedge N} (t_n)^+ \Big] = \ee\Big[\sum_{n=1}^{\tau} (t_n)^+ \Big].
\]
The above convergence results are applied to obtain the second equality below:
\begin{align*}
V(\mu^\alpha, \Sigma^\alpha, \mu^\beta, \Sigma^\beta)
&=
\sup_{\tau \ge 1, (u_n)}
\ee \Big[H(u_{(\tau)-1}) - \gamma \sum_{n=1}^{\tau} (t_n)^+ \Big]\\
&=
\sup_{\tau \ge 1, (u_n)} \lim_{N \to \infty}
\ee \Big[H(u_{(\tau \wedge N)-1}) - \gamma \sum_{n=1}^{\tau \wedge N} (t_n)^+ \Big]\\
&=
\sup_{\tau \ge 1, (u_n)} \liminf_{N \to \infty}
\ee \Big[H(u_{(\tau \wedge N)-1}) - \gamma \sum_{n=1}^{\tau \wedge N} (t_n)^+ \Big]\\
&\le
 \liminf_{N \to \infty} V_N (\mu^\alpha, \Sigma^\alpha, \mu^\beta, \Sigma^\beta)
= V_\infty (\mu^\alpha, \Sigma^\alpha, \mu^\beta, \Sigma^\beta),
\end{align*}
where we use that $\ee \Big[H(u_{(\tau \wedge N)-1}) - \gamma \sum_{n=1}^{\tau \wedge N} (t_n)^+ \Big] \le V_N (\mu^\alpha, \Sigma^\alpha, \mu^\beta, \Sigma^\beta)$ for any choice of admissible strategy $(\tau, (u_n))$. Hence, $V \le V_\infty$. In view of the opposite inequality obtained above, we have $V = V_\infty$. Lemma \ref{lem:upper_bound} guarantees that $|V| < \infty$. This completes the proof of Assertion 3.

Using monotonicity of $V_N$ in $N$ and $V_{N+1} = \cT V_N$, we have (omitting parameters of functions for clarity of presentation)
\begingroup\allowdisplaybreaks
\begin{align*}
V &= \lim_{N \to \infty} V_{N+1} = \sup_{N} V_{N+1} = \sup_{N} \cT V_N\\
&=
\sup_N \sup_{u \in \U}  \Big[ - \gamma \Tcost\Big(\bpsi(u)^T \mu^\beta, \bpsi(u)^T \Sigma^\beta \bpsi(u) + \sigma^2_T(u)\Big) \\
&\hspace{150pt}+ \ee \big[\max \big( ( \mu^\alpha_1)^T \bphi(u), V_N( \mu^\alpha_1,  \Sigma^\alpha_1,  \mu^\beta_1,  \Sigma^\beta_1) \big)\big]\Big]\\
&=
\sup_{u \in \U} \sup_N \Big[ - \gamma \Tcost\Big(\bpsi(u)^T \mu^\beta, \bpsi(u)^T \Sigma^\beta \bpsi(u) + \sigma^2_T(u)\Big) \\
&\hspace{150pt}+ \ee \big[\max \big( ( \mu^\alpha_1)^T \bphi(u), V_N( \mu^\alpha_1,  \Sigma^\alpha_1,  \mu^\beta_1,  \Sigma^\beta_1) \big)\big]\Big]\\
&=\sup_{u \in \U} \Big[ - \gamma \Tcost(\cdot )
+ \sup_{N} \ee \big[\max \big( ( \mu^\alpha_1)^T \bphi(u), V_N( \mu^\alpha_1,  \Sigma^\alpha_1,  \mu^\beta_1,  \Sigma^\beta_1) \big)\big]\Big]\\
&\mathop{=}_{\text{monotonicity }}
\sup_{u \in \U} \Big[ - \gamma \Tcost(\cdot )
+ \lim_{N \to \infty} \ee \big[\max \big( ( \mu^\alpha_1)^T \bphi(u), V_N( \mu^\alpha_1,  \Sigma^\alpha_1,  \mu^\beta_1,  \Sigma^\beta_1) \big)\big]\Big]\\
&=
\sup_{u \in \U} \Big[ - \gamma \Tcost(\cdot )
+ \ee \big[\max \big( ( \mu^\alpha_1)^T \bphi(u), V( \mu^\alpha_1,  \Sigma^\alpha_1,  \mu^\beta_1,  \Sigma^\beta_1) \big)\big]\Big],
\end{align*}%
\endgroup
where the last equality follows from pointwise convergence of $V_N$ to $V$ and the monotone convergence theorem whose assumptions are satisfied as
\[
\max \big( ( \mu^\alpha_1)^T \bphi(u), V_N( \mu^\alpha_1,  \Sigma^\alpha_1,  \mu^\beta_1,  \Sigma^\beta_1) \big)
\ge V_1 ( \mu^\alpha_1,  \Sigma^\alpha_1,  \mu^\beta_1,  \Sigma^\beta_1)
\]
and $E[|V_1( \mu^\alpha_1,  \Sigma^\alpha_1,  \mu^\beta_1,  \Sigma^\beta_1) |] < \infty$ by \eqref{eqn:9} and Lemma \ref{lem:A2}.
\end{proof}

\section{AutoBML settings for validation examples}\label{app:D}

Throughout all examples we set $\gamma = 0.16$. We also note that we use two value function maps for 1D examples (one with $\sigma_H = 0.05$ and $\sigma_T = 0.1$, and another one with $\sigma_H = 0.15$ and $\sigma_T = 0.1$) and one value function map for 2D examples with $\sigma_H = 0.15$ and $\sigma_T = 0.1$. Regressions in value function maps (see line 7 in Algorithm \ref{alg:lambda} and line 6 in Algorithm \ref{alg:vi}) were performed via Random Forest Regression.

\subsection{1D example with Synthetic data (Section \ref{subsec:syntetic})}

Basis functions:
\begin{itemize}
\item $H$: $\{1,u-0.5,(u-0.5)^2,(u-0.5)^3\}$
\item $T$: $\{1,u-0.5,(u-0.5)^2,(u-0.5)^3\}$
\end{itemize}

{\raggedleft Observation error:}
\begin{itemize}
\item $H$: $\sigma_H = 0.05$
\item $T$: $\sigma_T = 0.1$
\end{itemize}

{\raggedleft Method of fitting $\tilde\Lambda(\cdot)$:}
\begin{itemize}
\item \verb|smooth.spline| function in R, that fits B-splines according to \citep{cleveland1992}.
\end{itemize}

{\raggedleft Prior state $x_0$:}
\begin{itemize}
\item $\mu_\alpha$ = $(0.4, 0.1, -0.2, 0.1)$
\item $\mu_\beta$ = $(1, 1, 2, 2)$
\item $\Sigma_\alpha$ = $I_4$, ($I_n$ is $n\times n$ identity matrix)
\item $\Sigma_\beta$ = $(0.64, 4, 4, 4)\times I_4$
\end{itemize}

{\raggedleft Score (test accuracy) and Cost (running time in seconds) mapping:}
\begin{itemize}
\item $H$: $h = 2(\hat h - 0.5)$
\item $T$: $t = \hat t/0.6$
\end{itemize}

{\raggedleft Control $u \in [0,1]$ to hyperparameter $\ntrees$ mapping:}
\begin{itemize}
\item $\ntrees$: $\lfloor 99u_1 + 1\rfloor$
\end{itemize}

\subsection{1D example with Breast Cancer data (\textit{batch}) (Subsection \ref{sssec:mapcnnbatch})}

Basis functions:
\begin{itemize}
\item $H$: $\{1,u-0.5,(u-0.5)^2,(u-0.5)^3\}$
\item $T$: $\{1,u-0.5,(u-0.5)^2,(u-0.5)^3\}$
\end{itemize}

{\raggedleft Observation error:}
\begin{itemize}
\item $H$: $\sigma_H = 0.15$
\item $T$: $\sigma_T = 0.1$
\end{itemize}

{\raggedleft Method of fitting $\tilde\Lambda(\cdot)$:}
\begin{itemize}
\item \verb|smooth.spline| function in R, that fits B-splines according to \citep{cleveland1992}.
\end{itemize}

{\raggedleft Prior state $x_0$:}
\begin{itemize}
\item $\mu_\alpha$ = $(0.4,-0.1,-0.2,-0.1)$
\item $\mu_\beta$ = $(1,-1,2,-2)$
\item $\Sigma_\alpha$ = $I_4$, ($I_n$ is $n\times n$ identity matrix)
\item $\Sigma_\beta$ = $(0.64, 4, 4, 4)\times I_4$
\end{itemize}

{\raggedleft Score (test accuracy) and Cost (running time in seconds) mapping:}
\begin{itemize}
\item $H$: $h = (\hat h - 0.45)/0.35$
\item $T$: $t = \hat t/1.5$
\end{itemize}

{\raggedleft Control $u \in [0,1]$ to hyperparameter (\textit{batch}) mapping:}
\begin{itemize}
\item \textit{batch}: $\lfloor(200-10) u + 10\rfloor$
\end{itemize}

\subsection{1D example with Breast Cancer data (\textit{r})}

Basis functions:
\begin{itemize}
\item $H$: $\{1,u-0.5,(u-0.5)^2,(u-0.5)^3\}$
\item $T$: $\{1,u-0.5,(u-0.5)^2,(u-0.5)^3\}$
\end{itemize}

{\raggedleft Observation error:}
\begin{itemize}
\item $H$: $\sigma_H = 0.15$
\item $T$: $\sigma_T = 0.1$
\end{itemize}

{\raggedleft Method of fitting $\tilde\Lambda(\cdot)$:}
\begin{itemize}
\item \verb|smooth.spline| function in R, that fits B-splines according to \citep{cleveland1992}.
\end{itemize}

{\raggedleft Prior state $x_0$:}
\begin{itemize}
\item $\mu_\alpha$ = $(0.4,-0.1,-0.8,-0.1)$
\item $\mu_\beta$ = $(0.48,0, 0, 0)$
\item $\Sigma_\alpha$ = $I_4$, ($I_n$ is $n\times n$ identity matrix)
\item $\Sigma_\beta$ = $(0.8^2,0,0,0)\times I_4$
\end{itemize}

{\raggedleft Score (test accuracy) and Cost (running time in seconds) mapping:}
\begin{itemize}
\item $H$: $h = (\hat h - 0.45)/0.35$
\item $T$: $t = \hat t/1.5$
\end{itemize}

{\raggedleft Control $u \in [0,1]$ to hyperparameter (\textit{r}) mapping:}
\begin{itemize}
\item $r$: $\exp[(\log(0.1)-\log(0.00001))u + \log(0.00001)]$
\end{itemize}

\subsection{2D example with Breast Cancer data}

Basis functions:
\begin{itemize}
\item $H$: $\{(u_1 - 0.5)^k\}_{k=0}^4$, $\{(u_2 - 0.5)^k\}_{k=1}^4$, $(u_1 - 0.5) \cdot (u_2 - 0.5)$
\item $T$: $\{(u_1 - 0.5)^k\}_{k=0}^4$, $\{(u_2 - 0.5)^k\}_{k=1}^4$, $(u_1 - 0.5) \cdot (u_2 - 0.5)$
\end{itemize}

{\raggedleft Observation error:}
\begin{itemize}
\item $H$: $\sigma_H = 0.15$
\item $T$: $\sigma_T = 0.1$
\end{itemize}

{\raggedleft Method of fitting $\tilde\Lambda(\cdot)$:}
\begin{itemize}
\item \verb|Tps| function in R from package \verb|fields|, that fits thin-plate splines, see \citep{green1993} for more detailed analysis on nonparametric regression.
\end{itemize}

{\raggedleft Prior state $x_0$:}
\begin{itemize}
\item $\mu_\alpha$ = $(0.4, 0.3, 0.4, 0,0, 0.2, -0.4, 0, 0,0)$
\item $\mu_\beta$ = $(3, 0, 0, 0,0, -3.5, 0.45, 0, 0.5,0)$
\item $H(u_1,u_2) = 0.4 + 0.3\cdot u_1 +  0.2\cdot u_2 - 0.4\cdot u_1^2 - 0.4\cdot u_2^2$
\item $T(u_1,u_2) = 3 - 3.5\cdot u_2 + 0.45 \cdot u_2^2 + 0.5\cdot u_2^4$
\item $\Sigma_\alpha$ =  $0.6\times I_{10}$
\item Create $\Sigma_{\beta}$ = $0.6 \times I_{10}$, set $\sigma_{\beta ,9,9} = 0.001$
\end{itemize}

{\raggedleft Score (test accuracy) and Cost (minutes) scaling:}
\begin{itemize}
\item $H$: $h_{\min}=0.45, h_{\max}=0.8$
\item $T$: $t_{\min}=0, t_{\max}=7.5$
\end{itemize}

{\raggedleft Control $u \in [0,1]\times[0,1]$ to hyperparameters (\textit{learning rate} and \textit{batch size}) mappings:}
\begin{itemize}
\item $r$: $\exp[(\log(0.1)-\log(0.00001))u_1 + \log(0.00001)]$
\item \textit{batch}: $\lfloor(200-10) u_2 + 10\rfloor$
\end{itemize}

\subsection{1D example with Higgs data}

Basis functions:
\begin{itemize}
\item $H$: $\{1,u-0.5,(u-0.5)^2,(u-0.5)^3\}$
\item $T$: $\{1,u-0.5,(u-0.5)^2,(u-0.5)^3\}$
\end{itemize}

{\raggedleft Observation error:}
\begin{itemize}
\item $H$: $\sigma_H = 0.05$
\item $T$: $\sigma_T = 0.1$
\end{itemize}

{\raggedleft Method of fitting $\tilde\Lambda(\cdot)$:}
\begin{itemize}
\item \verb|smooth.spline| function in R, that fits B-splines.
\end{itemize}

{\raggedleft Prior state $x_0$:}
\begin{itemize}
\item $\mu_\alpha$ = $(0.5, 0.3, -0.2, 0)$
\item $\mu_\beta$ = $(1.4, 4, 6, 6)$
\item $\Sigma_\alpha$ = $I_4$, ($I_n$ is $n\times n$ identity matrix)
\item $\Sigma_\beta$ = $I_4$
\end{itemize}

{\raggedleft Score (test accuracy) and Cost (minutes) scaling:}
\begin{itemize}
\item $H$: $h_{\min}=0.6, h_{\max}=0.8$
\item $T$: $t_{\min}=0.035, t_{\max}=5$
\end{itemize}

{\raggedleft Control $u \in [0,1]$ to hyperparameter (\textit{sample size}) mapping:}
\begin{itemize}
\item \textit{sample size}: $\exp[\log_{15}(0.00001) - u_1 \log_{15}(0.00001)]$
\end{itemize}

\subsection{1D example with Intel Images data}

Basis function:
\begin{itemize}
\item $H$: $\{1,u-0.5,(u-0.5)^2,(u-0.5)^3\}$
\item $T$: $\{1,u-0.5,(u-0.5)^2,(u-0.5)^3\}$
\end{itemize}

{\raggedleft Observation error:}
\begin{itemize}
\item $H$: $\sigma_H = 0.15$
\item $T$: $\sigma_T = 0.1$
\end{itemize}

{\raggedleft Method of fitting $\tilde\Lambda(\cdot)$:}
\begin{itemize}
\item \verb|smooth.spline| function in R, that fits B-splines.
\end{itemize}

{\raggedleft Prior state $x_0$:}
\begin{itemize}
\item $\mu_\alpha$ = $(0.9, 0.4, -2, 0)$
\item $\mu_\beta$ = $(0.8, 1.5, 0, 0)$
\item $\Sigma_\alpha$ = $I_4$, ($I_n$ is $n\times n$ identity matrix)
\item $\Sigma_\beta$ = $I_4$
\end{itemize}

{\raggedleft Score (test accuracy) and Cost (minutes) scaling:}
\begin{itemize}
\item $H$: $h_{\min}=0, h_{\max}=1$
\item $T$: $t_{\min}=0.1, t_{\max}=12$
\end{itemize}

{\raggedleft Control $u \in [0,1]$ to hyperparameter (\textit{epoch}) mapping:}
\begin{itemize}
\item \textit{epoch}: $ \lfloor34u_1 + 1\rfloor$
\end{itemize}

\subsection{1D example with Credit Card Fraud data}

Basis functions:
\begin{itemize}
\item $H$: $\{1,u-0.5,(u-0.5)^2,(u-0.5)^3\}$
\item $T$: $\{1,u-0.5,(u-0.5)^2,(u-0.5)^3\}$
\end{itemize}

{\raggedleft Observation error:}
\begin{itemize}
\item $H$: $\sigma_H = 0.15$
\item $T$: $\sigma_T = 0.1$
\end{itemize}

{\raggedleft Method of fitting $\tilde\Lambda(\cdot)$:}
\begin{itemize}
\item \verb|smooth.spline| function in R, that fits B-splines.
\end{itemize}

{\raggedleft Prior state $x_0$:}
\begin{itemize}
\item $\mu_\alpha$ = $(0.9, 0, -2, 0)$
\item $\mu_\beta$ = $(1.1, 0.6, 0, 0)$
\item $\Sigma_\alpha$ = $I_4$, ($I_n$ is $n\times n$ identity matrix)
\item $\Sigma_\beta$ = $I_4$
\end{itemize}

{\raggedleft Score (generalized F-score, $\beta=6$) and Cost (minutes) scaling:}
\begin{itemize}
\item $H$: $h_{\min}=0, h_{\max}=0.75$
\item $T$: $t_{\min}=5, t_{\max}=25$
\end{itemize}

{\raggedleft Control $u \in [0,1]$ to hyperparameter (\textit{scale}) mapping:}
\begin{itemize}
\item \textit{scale}: $ 9u_1 + 1$
\end{itemize}

\subsection{2D example with Credit Card Fraud data}

Basis functions:
\begin{itemize}
\item $H$: $\{(u_1 - 0.5)^k\}_{k=0}^4$, $\{(u_2 - 0.5)^k\}_{k=1}^4$, $(u_1 - 0.5) \cdot (u_2 - 0.5)$
\item $T$: $\{(u_1 - 0.5)^k\}_{k=0}^4$, $\{(u_2 - 0.5)^k\}_{k=1}^4$, $(u_1 - 0.5) \cdot (u_2 - 0.5)$
\end{itemize}

{\raggedleft Observation error:}
\begin{itemize}
\item $H$: $\sigma_H = 0.15$
\item $T$: $\sigma_T = 0.1$
\end{itemize}

{\raggedleft Method of fitting $\tilde\Lambda(\cdot)$:}
\begin{itemize}
\item \verb|Tps| function in R from package \verb|fields|, that fits thin-plate splines.
\end{itemize}

{\raggedleft Prior state $x_0$:}
\begin{itemize}
\item $\mu_\alpha$ = $(0.8, 0, -0.8, 0, 0, 0, -0.8, 0, 0, 0)$
\item $\mu_\beta$ = $(1.325, 0.75, 0.25, 0, 0, 0.75, 0.25, 0, 0, 0)$
\item $H(u_1,u_2) = 0.8 - 0.8\cdot u_1^2 - 0.8\cdot u_2^2$
\item $T(u_1,u_2) = 1.325 + 0.75\cdot u_1 + 0.25 \cdot u_1^2 + 0.75\cdot u_2 + 0.25\cdot u_2^2 $
\item Generate $\Sigma_\alpha$ =  $(4,0.6,0.62,0.65,0.70,0.6,0.62,0.65,0.70,0.6) \times I_{10}$, and set $\Sigma_\alpha [1,2:10] = \Sigma_\alpha [2:10,1] = -0.1$.
\item Create $\Sigma_{\beta} = I_{10}$
\end{itemize}

{\raggedleft Score (generalized F-score, $\beta=6$) and Cost (minutes) scaling:}
\begin{itemize}
\item $H$: $h_{\min}=0.45, h_{\max}=0.8$
\item $T$: $t_{\min}=0, t_{\max}=7.5$
\end{itemize}

{\raggedleft Control $u \in [0,1]\times[0,1]$ to hyperparameters (\textit{scale} and \textit{code}) mappings:}
\begin{itemize}
\item \textit{scale}: $ 9u_1 + 1$
\item \textit{code}: $ \text{round}(9u_2 + 1)$
\end{itemize}

\section{Model architectures}\label{app:arch}

We provide snippets of a Python code defining architectures of neural networks used in Section \ref{sec:validation}.

\subsection{CNN Breast Cancer example}\label{app:breastcancerarch}

\begin{verbatim}
    dropConv = 0.2 
    scale = 1
    
    # Create the Model Architecture
            
    kernel_size = (3,3)
    pool_size= (2,2)
    
    first_filters = 8*scale
    second_filters = 16*scale
    third_filters = 32*scale
    
    dropout_conv = dropConv
    
    model = Sequential()
    model.add(Conv2D(first_filters, kernel_size, activation = 'relu', 
                     input_shape = (IMAGE_SIZE, IMAGE_SIZE, 3)))
    model.add(MaxPooling2D(pool_size = pool_size)) 
    model.add(Dropout(dropout_conv))
    
    model.add(Conv2D(second_filters, kernel_size, activation ='relu'))
    model.add(MaxPooling2D(pool_size = pool_size))
    model.add(Dropout(dropout_conv))
    
    model.add(Conv2D(third_filters, kernel_size, activation ='relu'))
    model.add(MaxPooling2D(pool_size = pool_size))
    model.add(Dropout(dropout_conv))
    
    model.add(Flatten())
    model.add(Dense(16*scale, activation = "relu"))
    model.add(Dense(2, activation = "softmax"))
\end{verbatim}

\subsection{Plain and Resnet50 models for Intel data}\label{app:intelarch}

\subsection{Plain}
\begin{verbatim}  
  model = Models.Sequential()
  model.add(Layers.Conv2D(200,kernel_size=(3,3),activation='relu', \
                          input_shape=(150,150,3)))
  model.add(Layers.Conv2D(200,kernel_size=(3,3),activation='relu'))
  model.add(Layers.Conv2D(180,kernel_size=(3,3),activation='relu'))
  model.add(Layers.MaxPool2D(pool_size=(5,5), padding="same"))
  model.add(Layers.Conv2D(180,kernel_size=(3,3),activation='relu', padding="same"))
  model.add(Layers.Conv2D(140,kernel_size=(3,3),activation='relu', padding="same"))
  model.add(Layers.Conv2D(100,kernel_size=(3,3),activation='relu', padding="same"))
  model.add(Layers.Conv2D(50,kernel_size=(3,3),activation='relu', padding="same"))
  model.add(Layers.MaxPool2D(pool_size=(5,5), padding="same"))
  model.add(Layers.Flatten())
  model.add(Layers.Dense(180,activation='relu'))
  model.add(Layers.Dense(100,activation='relu'))
  model.add(Layers.Dense(50,activation='relu'))
  model.add(Layers.Dropout(rate=0.5))
  model.add(Layers.Dense(6,activation='softmax'))
\end{verbatim}

\subsection{Resnet50}
\begin{verbatim}
  resnet = ResNet50(include_top=False, weights="imagenet", input_shape=(150, 150, 3))
  resnet.trainable = False
  
  model = Models.Sequential()
  model.add(resnet)
  model.add(Layers.Conv2D(200,kernel_size=(3,3),activation='relu'))
  model.add(Layers.Conv2D(180,kernel_size=(3,3),activation='relu'))
  model.add(Layers.MaxPool2D(pool_size=(5,5), padding="same"))
  model.add(Layers.Conv2D(180,kernel_size=(3,3),activation='relu', padding="same"))
  model.add(Layers.Conv2D(140,kernel_size=(3,3),activation='relu', padding="same"))
  model.add(Layers.Conv2D(100,kernel_size=(3,3),activation='relu', padding="same"))
  model.add(Layers.Conv2D(50,kernel_size=(3,3),activation='relu', padding="same"))
  model.add(Layers.MaxPool2D(pool_size=(5,5), padding="same"))
  model.add(Layers.Flatten())
  model.add(Layers.Dense(180,activation='relu'))
  model.add(Layers.Dense(100,activation='relu'))
  model.add(Layers.Dense(50,activation='relu'))
  model.add(Layers.Dropout(rate=0.5))
  model.add(Layers.Dense(6,activation='softmax'))
\end{verbatim}

\bibliographystyle{apalike}
\bibliography{autoML}

\end{document}